\documentclass[10pt,journal,compsoc]{IEEEtran}

\ifCLASSOPTIONcompsoc
  \usepackage[nocompress]{cite}
\else
  \usepackage{cite}
\fi

\ifCLASSINFOpdf
\else
\fi
\usepackage{hyperref}
\hypersetup{hidelinks}
\usepackage{url}            
\usepackage{booktabs}       
\usepackage{amsfonts}       
\usepackage{nicefrac}       
\usepackage{microtype}      
\usepackage{xcolor}         

\usepackage{amssymb}
\usepackage{amsmath,mathrsfs,dsfont}
\usepackage{nicefrac}
\usepackage[algo2e]{algorithm2e}
\usepackage{algorithmic}
\usepackage{algorithm}

\hypersetup{colorlinks=true,citecolor=blue,linkcolor=blue}

\usepackage{booktabs}

\usepackage{color}
\usepackage{enumitem}

\usepackage{array,cite}
\usepackage{graphicx,tikz}
\usepackage[mathscr]{euscript}
\usepackage{amsthm}
\usepackage{cite}

\usepackage{bm}
\usepackage{bbm}
\usepackage{color}

\usepackage{color}
\usepackage{epstopdf}
\usepackage{subcaption}
\usepackage{cleveref}
\usepackage{thmtools}
\usepackage{thm-restate}

\usepackage{bbding}
\usepackage{mathtools}
\allowdisplaybreaks
\usepackage{multirow}
\usepackage{adjustbox}
\usepackage{slashbox}
\usepackage{wrapfig}
\DeclareUnicodeCharacter{FB01}{fi}
\DeclareUnicodeCharacter{FB02}{fi}

\graphicspath{{illustrations/}}

\newcounter{optproblem}

\newtheoremstyle{mytheoremstyle} 
    {\topsep}                    
    {\topsep}                    
    {\normalfont}                
    {}                           
    {\bfseries}                   
    {.}                          
    {.5em}                       
    {}  

\theoremstyle{mytheoremstyle}
\newtheorem{theorem}{Theorem}[section]

\newtheorem*{theorem*}{Theorem}
\newtheorem*{lemma*}{Lemma}
\newtheorem*{remark*}{Remark}

\newtheorem{lemma}[theorem]{Lemma}

\theoremstyle{mytheoremstyle}

\theoremstyle{remark}

\DeclareMathAlphabet{\pazocal}{OMS}{zplm}{m}{n}
\DeclareMathAlphabet{\mathpzc}{OMS}{pzc}{m}{it}

\setlist[itemize]{leftmargin=*}

\renewcommand{\hat}{\widehat}

     \def\RR{\mathbb{R}}

 \def\cB{{\cal  B}}
 \def\cC{{\cal  C}}

 \def\cI{{\cal  I}}

 \def\cL{{\cal  L}}
 \def\cM{{\cal  M}}

 \def\cV{{\cal  V}}
 \def\cW{{\cal  W}}

\def\+#1{\mathcal{#1}}
\def\-#1{\textup{#1}}

\def\set#1{\left\{ #1 \right\}}
\def\pth#1{\left( #1 \right)}

\def\abth#1{\left | #1 \right |}

\def\defeq {\coloneqq}

\def \longmid {\,\middle\vert\,}

\newcommand{\La}{\left\langle\kern-0.64ex\left\langle}
\newcommand{\Ra}{\right\rangle\kern-0.64ex\right\rangle}

\def\Norm#1#2{{\left\vert\kern-0.4ex\left\vert\kern-0.4ex\left\vert #1
    \right\vert\kern-0.4ex\right\vert\kern-0.4ex\right\vert}_{#2}}
\def\norm#1#2{{\left\|#1\right\|}_{#2}}

\def\ltwonorm#1{\norm{#1}{2}}

\newcommand{\1}{{\rm 1}\kern-0.25em{\rm I}}
\def\indict#1{{\rm 1}\kern-0.25em{\rm I}_{\set{#1}}}

\def\set#1{\left\{#1\right\}}

\def \Pr {\textup{Pr}}
\newcommand{\Prob}[1]{\Pr\left[#1\right]}

\newcommand{\beq}{\begin{equation}}
\newcommand{\eeq}{\end{equation}}
\newcommand{\beqa}{\begin{eqnarray}}
\newcommand{\eeqa}{\end{eqnarray}}
\newcommand{\beqas}{\begin{eqnarray*}}
\newcommand{\eeqas}{\end{eqnarray*}}
\def\bal#1\eal{\begin{align}#1\end{align}}
\def\bals#1\eals{\begin{align*}#1\end{align*}}
\def\bsal#1\esal{\begin{small}\begin{align}#1\end{align}\end{small}}
\def\bsals#1\esals{\begin{small}\begin{align*}#1\end{align*}\end{small}}
\def\bsfal#1\esfal{\begin{small}\begin{flalign}#1\end{flalign}\end{small}}

\begin{document}

%
\title{Learning Informative Attention Weights for Person Re-Identification}
%
%
%
%

\author{Yancheng~Wang,
        Nebojsa Jojic,
        Yingzhen~Yang
\IEEEcompsocitemizethanks{\IEEEcompsocthanksitem Yancheng~Wang, Yingzhen Yang are with the School of Computing and
Augmented Intelligence, Arizona State University, Tempe, AZ, 85281.\protect\\
E-mail: ywan1053@asu.edu, yingzhen.yang@asu.edu
\IEEEcompsocthanksitem Nebojsa Jojic is with Microsoft Research, Redmond, WA, 90852. \protect\\
E-mail:
jojic@microsoft.com
}
}

\IEEEtitleabstractindextext{%
\begin{abstract}
Attention mechanisms have been widely used in deep learning, and recent efforts have been devoted to incorporating attention modules into deep neural networks (DNNs) for person Re-Identification (Re-ID) to enhance their discriminative feature learning capabilities.
Existing attention modules, including self-attention and channel attention, learn attention weights that quantify the importance of feature tokens or feature channels.
However, existing attention methods do not explicitly ensure that the attention weights are informative for predicting the identity of the person in the input image, and may consequently introduce noisy information from the input image.
To address this issue, we propose a novel method termed Reduction of Information Bottleneck loss (RIB), motivated by the principle of the Information Bottleneck (IB).
A novel distribution-free and efficient variational upper bound for the IB loss (IBB), which can be optimized by standard SGD, is derived and incorporated into the training loss of the RIB models.
RIB is applied to DNNs with self-attention modules through a novel Differentiable Channel Selection Attention module, or DCS-Attention, that selects the most informative channels for computing attention weights, leading to competitive models termed RIB-DCS.
RIB is also incorporated into DNNs with existing channel attention modules to promote the learning of informative channel attention weights, leading to models termed RIB-CA.
Both RIB-DCS and RIB-CA are applied to fixed neural network backbones and learnable backbones with Differentiable Neural Architecture Search (DNAS).
Extensive experiments on multiple person Re-ID benchmarks show that RIB significantly enhances the prediction accuracy of DNNs for person Re-ID, even for the occluded person Re-ID.
Extensive experiment results demonstrate the effectiveness of RIB in learning discriminative features which are informative for identifying person identities.
The code of our work is available at \url{https://github.com/Statistical-Deep-Learning/RIB-ReID}.
\end{abstract}

}

\maketitle

\IEEEdisplaynontitleabstractindextext

%
\IEEEpeerreviewmaketitle

\IEEEraisesectionheading{\section{Introduction}\label{sec:introduction}}
\IEEEPARstart{A}{ttention} methods, including self-attention and channel attention, have recently attracted increasing interest in computer vision tasks, such as person re-identification (Re-ID).
Existing self-attention modules~\cite{wang2018non, zhang2020relation, chen2019abd, he2021transreid} compute attention weights by modeling pairwise affinities among feature tokens, while existing channel attention modules assign adaptive weights to feature channels~\cite{hu2018squeeze, CBAM, MCA}. In both cases, the computation of attention weights is critical to their success.
However, existing attention methods, including self-attention and channel attention, do not explicitly ensure that attention weights are informative for predicting the class label, which is the person identity in person Re-ID.
Such an issue is illustrated by the Grad-CAM visualization of an attention-based vision transformer model, SPT~\cite{tan2024occluded}, in Fig.~\ref{fig:cam_occ}(a)-(b), which illustrates that the model largely focuses on occlusion regions unrelated to the identity of the person in the image for the occluded person Re-ID task.
For another example, the heatmap visualization of attention weights in Fig.~\ref{fig:cam_occ}(c)-(d) illustrates that even feature tokens from the regions that occlude the person receive high attention weights for a query token from the boundary of the arm of the person. Similarly, the Grad-CAM visualization in Fig.~\ref{fig:grad_cam_1}(c) illustrates that the representative baselines in self-attention, TransReID~\cite{he2021transreid}, and channel attention, SE~\cite{hu2018squeeze}, either mistakenly focus on the background areas (in the bottom figure) or miss the important human body regions (in the top figure), potentially hurting the performance of person Re-ID. These examples show that existing attention-based methods, including self-attention and channel attention, lack a principled information-theoretic approach for learning informative attention weights so that higher attention weights are assigned to more informative features relevant to person identity for person Re-ID.

\begin{figure}[!t]
\begin{center}
\includegraphics[width=0.925\columnwidth]{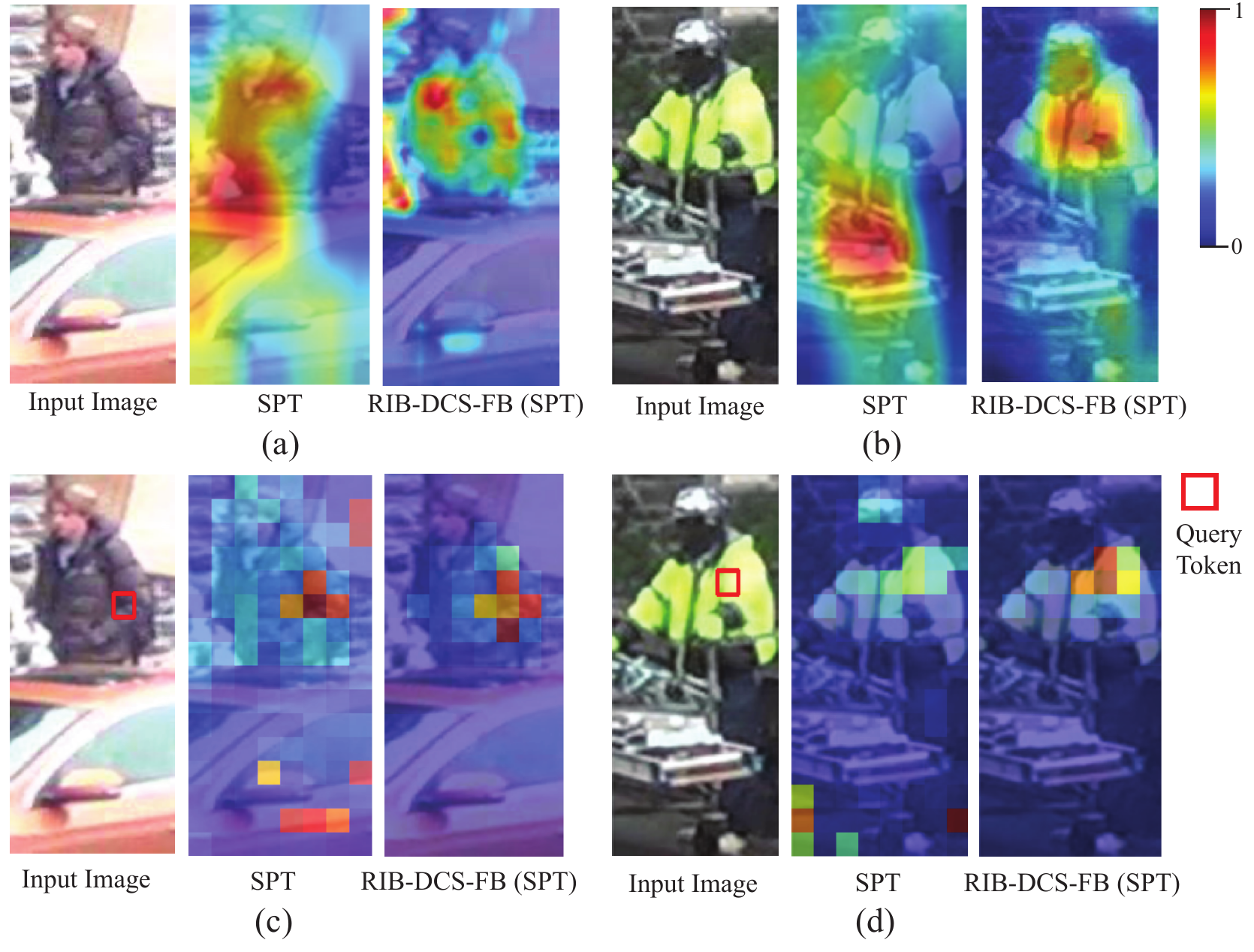}
\end{center}
\vspace{-4mm}
\caption{Figures (a)-(b) illustrate the Grad-CAM visualization for two images from the occluded person Re-ID dataset, Occluded-Duke \cite{miao2019pose}, for an attention-based model, SPT~\cite{tan2024occluded}, and RIB-DCS-FB (SPT).
Figures (c)-(d) illustrate the heatmaps of the attention weights corresponding to a query token computed from the first transformer block in SPT and RIB-DCS-FB (SPT). The query token for both examples is selected from the boundary of the arm of the person in the images, which is critical for identifying a person.}
\label{fig:cam_occ}
\vspace{-5mm}
\end{figure}

\begin{figure*}[!htb]
\vspace{-3mm}
\begin{center}
 \includegraphics[width=0.925\linewidth]{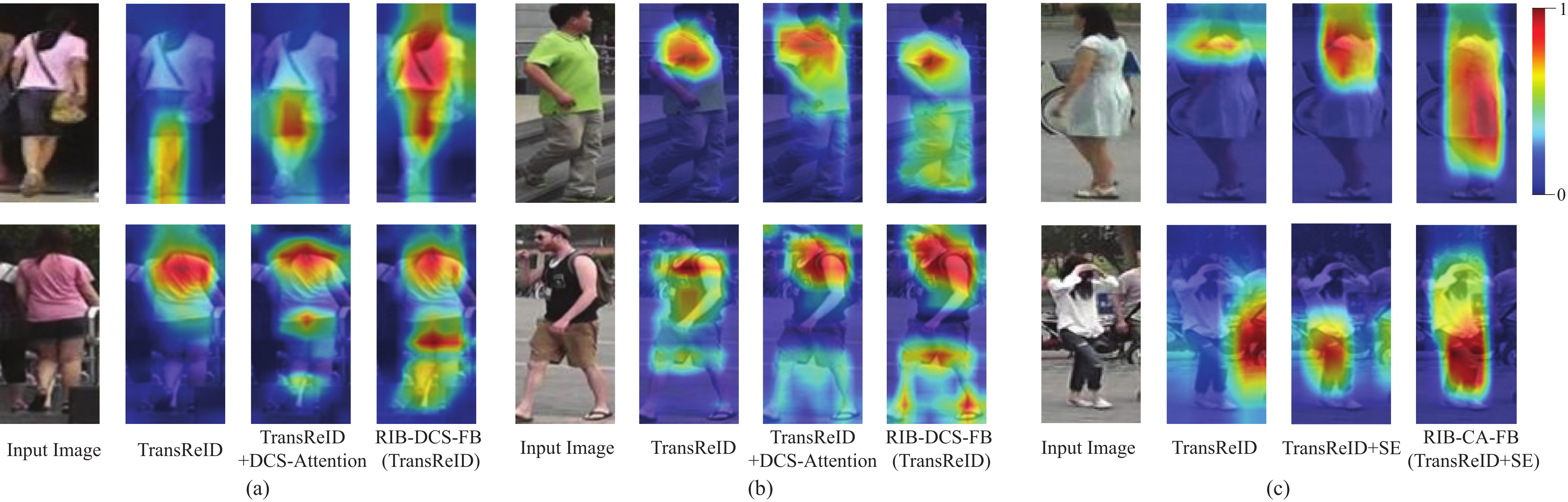}
\end{center}
\vspace{-5mm}
\caption{Figures (a)-(b) illustrate the Grad-CAM heatmaps for TransReID, TransReID+DCS-Attention, and RIB-DCS-FB (TransReID). TransReID+DCS-Attention replaces the self-attention modules in TransRe-ID with the DCS-Attention modules.
Figure (c) illustrates the Grad-CAM heatmaps for TransReID, TransReID+SE, and RIB-CA-FB (TransReID+SE). The SE attention modules~\cite{hu2018squeeze} are inserted after each transformer block in TransReID+SE. More visualization results are deferred to Section~\ref{sec:grad_cam} of the supplementary.
}
\vspace{-5mm}
\label{fig:grad_cam_1}
\end{figure*}
To address this issue, we propose to learn informative attention weights in existing self-attention modules and channel attention modules so that the resultant DNNs learn discriminative features for person Re-ID.
By learning informative attention weights, DNNs assign higher attention weights to feature tokens in self-attention modules and to feature channels in channel attention modules which contribute more to the discriminative task of person Re-ID.
Our proposed Reduction of Information Bottleneck loss (RIB) achieves this goal by reducing the Information Bottleneck (IB) loss to be detailed later in this section, inspired by the IB  principle~\cite{NaftaliIB}.  RIB is applied to DNNs with self-attention modules through a novel Differentiable Channel Selection Attention module, or DCS-Attention, which selects more informative channels for learning more informative attention weights than vanilla self-attention modules, leading to competitive models termed RIB-DCS. RIB is also incorporated into DNNs with existing channel attention modules to learn more informative channel attention weights than existing channel attention modules, leading to RIB-CA. Both RIB-DCS and RIB-CA are integrated into Fixed Backbones (FB) and learnable backbones with Differentiable Neural Architecture Search (DNAS), leading to RIB-DCS-FB, RIB-DCS-DNAS, RIB-CA-FB, and RIB-CA-DNAS.
As illustrated in Fig.~\ref{fig:cam_occ}(a)-(b), RIB-DCS-FB (SPT) focuses on the important parts of human figures instead of the occlusion regions in contrast to the baseline self-attention model, SPT.
Moreover, compared to SPT, Fig.~\ref{fig:cam_occ}(c)-(d) illustrates that RIB-DCS-FB (SPT) assigns higher attention weights to informative and semantically relevant feature tokens from the boundary of the arm of the person for a query that is nearly from the same location. We then explain the detailed motivation for applying RIB to self-attention and channel attention.

\noindent \textbf{Motivation of RIB for Self-Attention and Channel Attention.}
Self-attention~\cite{he2021transreid, zhang2020relation} can be formulated as $\textup{Output} = \sigma(QK^{\top})V$, where $Q,K,V \in \RR^{N \times D}$ denote the query, key, and value, respectively, with $N$ being the number of feature tokens and $D$ being the input channel number. $\sigma(\cdot)$ is an operator, such as Softmax. $A  = \sigma(QK^{\top}) \in \RR^{N \times N}$ is the attention weights which model the weighted aggregation of the feature tokens in $V$.
By learning informative attention weights $A$, DNNs assign higher attention weights to feature tokens in $V$ which are more informative and semantically relevant to the identity of the person.
For example, feature tokens corresponding to human body regions in Fig.~\ref{fig:cam_occ}(c)-(d) should receive higher attention weights than those from occlusions or background.
To this end, we propose Differentiable Channel Selection Attention, or DCS-Attention, which computes more informative attention weights $A$ than vanilla self-attention by using only selected informative channels, which are the selected columns of $Q$ and the corresponding selected rows of $K^{\top}$, to compute $A = \sigma(QK^{\top})$, as illustrated in Fig.~\ref{fig:pipeline}(a).
Learning informative attention weights to learn discriminative features can be understood through the Information Bottleneck (IB) principle. Let $X$ be the input features, and $F$ be the learned features by the network. Let $Y$ be the ground truth training class labels, which define the identity of the person for person Re-ID. The principle of IB is to increase the mutual information between $F$ and $Y$ while reducing the mutual information between $F$ and $X$. That is, the IB principle encourages reduction of the IB loss, $I(F,X) - I(F,Y)$, where $I(\cdot,\cdot)$ denotes mutual information modeling the correlation of its input variables.
We remark that by reducing the IB loss with RIB, the DNNs are encouraged to use more informative channels to compute more informative attention weights in self-attention, so that the attention output is the aggregation of the feature tokens in the value $V$ which are more correlated with the person identity and less correlated with the input image containing background or occlusions.
Table~\ref{tab:ablation-IB-loss} in Section~\ref{sec:ablation_IB_loss} demonstrates that replacing standard self-attention with DCS-Attention in a DNN leads to a lower IB loss and improved person Re-ID performance without explicitly reducing the IB loss in training.
By explicitly reducing the IB loss with RIB to learn more informative attention weights, the person Re-ID performance is further improved, and the IB loss is reduced to an even lower level, which indicates the effectiveness of RIB. The key difference of our hard channel selection in DCS-Attention from existing channel attention methods, such as SE~\cite{hu2018squeeze}, CBAM~\cite{CBAM}, and MCA~\cite{MCA}, along with its advantage, is explained in Section~\ref{sec:sup_hard-channel-selection} of the supplementary.

Similar to existing self-attention modules, existing channel attention modules typically compute attention weights from activation statistics or heuristic gating functions~\cite{hu2018squeeze,CBAM,MCA}, without explicitly ensuring that channels with higher attention weights are more informative for person Re-ID. Under the RIB framework, RIB-CA learns informative channel attention weights for SE~\cite{hu2018squeeze}, CBAM~\cite{CBAM}, and MCA~\cite{MCA}, so that more discriminative feature channels receive higher attention weights. RIB is particularly effective for occluded person Re-ID, where reducing the correlation between the learned features and the inputs with occlusions while increasing the correlation between the learned features and the class labels enables the network to learn discriminative features that are robust to occlusions, as evidenced by Fig.~\ref{fig:cam_occ} and Table~\ref{tab:occluded_duke}.

\noindent \textbf{Contributions.} The main contributions of this paper are presented as follows.

First, we propose a novel distribution-free and efficient Information Bottleneck reduction framework, termed RIB. RIB aims to learn informative attention weights for self-attention and channel attention modules by explicitly reducing the IB loss in the training process of the DNNs with attention modules for person Re-ID.
In order to reduce the IB loss, we present a novel and provable variational upper bound for the IB loss, termed IBB, which can be optimized by standard SGD algorithms. The IB loss is then reduced by reducing its upper bound, the IBB.
Different from existing upper bounds for the IB loss, such as VIB~\cite{VIB-DaiZGW18, VIB-SrivastavaDGAA21}, which impose an unrealistic Gaussian assumption on the hidden features of DNNs, and APIB~\cite{IB-lasso-APIB}, which reduces only an approximation to the IB loss, RIB directly minimizes a variational upper bound for the IB loss, IBB, without introducing any distributional assumptions on the hidden features.
Moreover, the proposed IBB is computationally efficient with a computational complexity of $\Theta(n C T_0)$, where $C$ is the number of classes, $n$ is the number of training samples, and $T_0$ denotes the computational complexity of a forward and backward pass of the neural network with respect to each training sample.
In contrast, the upper bound for the mutual information used to bound the IB loss in CLUB~\cite{CLUB}, albeit not requiring distributional assumptions on the hidden features of DNNs, requires a substantially higher computational complexity of $\Theta(n^2 T_0)$ since $n \gg C$.
A composite loss, which combines the IBB and the regular cross-entropy/triplet losses, is used to train DNNs with attention modules for person Re-ID.
Table~\ref{tab:ablation-IB-loss} in Section~\ref{sec:ablation_IB_loss} demonstrates that the IB loss of DNNs trained for person Re-ID can be reduced by optimizing such a composite loss.
Table~\ref{tab:compare-IBB-VIB-APIB} in Section~\ref{sec:ablation_IBB} demonstrates that the models using IBB achieve substantially better performance than the models using CLUB, VIB, and APIB.
We remark that our RIB framework, as an independent contribution, can be applied to broader discriminative tasks beyond person Re-ID.

Second, RIB is applied to DNNs with self-attention or channel attention modules to learn informative attention weights. Under the general framework of RIB, we propose a novel DCS-Attention module which selects more informative channels for learning more informative self-attention weights than existing self-attention modules, leading to RIB-DCS. In addition, RIB is applied to DNNs with existing channel attention modules, including SE~\cite{hu2018squeeze}, CBAM~\cite{CBAM}, and MCA~\cite{MCA}, to learn informative channel attention weights, leading to RIB-CA.
Both RIB-DCS and RIB-CA are applied to Fixed Backbones (FB) and learnable backbones with Differentiable Neural Architecture Search (DNAS).
Extensive experiments in Section~\ref{sec::results-sw-DCS} and Section~\ref{sec:ChannelAttn_Results} demonstrate that both RIB-DCS and RIB-CA models significantly outperform the baseline models without RIB on public person Re-ID benchmarks.
Visualization results in Fig.~\ref{fig:cam_occ}, Fig.~\ref{fig:grad_cam_1}, and Fig.~\ref{fig:cam}-\ref{fig:RIB-CA_cam} in Section~\ref{sec:grad_cam} of the supplementary illustrate the advantages of RIB-DCS and RIB-CA compared to existing self-attention and channel attention modules.
The RIB models also show superior performance for more challenging tasks, such as occluded person Re-ID, cross-domain person Re-ID, and self-supervised person Re-ID, as evidenced in Section~\ref{sec:occluded_duke} and Section~\ref{sec:cross_dataset}-\ref{sec:SSL} in the supplementary.

\vspace{-2mm}
\section{Related Works}
\subsection{Person Re-identification and Attention Modules}
\label{sec:reid_attention}
Person Re-IDentification (Re-ID) focuses on identifying and matching an individual across various distinct camera views. Both supervised~\cite{autoreid, zhang2023pha, he2021transreid} and unsupervised~\cite{chen2023jaccard, yang2024shallow} methods have been developed by utilizing DNNs.

\noindent\textbf{Related Works about Attention Methods for Person Re-ID.} Following the success of self-attention models~\cite{CBAM, ZhaoJK2020-SA-recognition} in computer vision, recent works seek to incorporate attention mechanisms into DNNs designed for person Re-ID. Some studies \cite{zhao2017spindle, zheng2019pose} also explore using external clues of human semantic annotations as guidance to enforce attention.
Following the success of self-attention \cite{vaswani2017attention} and its adaptation \cite{wang2018non}, recent studies \cite{zhang2020relation, chen2019abd, autoreid} in person Re-ID adopt self-attention modules or non-local attention blocks to capture global feature dependencies. In addition to incorporating attention modules into Convolutional Neural Networks (CNNs) designed for person Re-ID~\cite{zhang2020relation, chen2019abd}, recent works~\cite{zhang2023pha, he2021transreid} also seek to design attention-based Visual Transformers (ViTs) for person Re-ID. To further enhance the performance of DNNs for person Re-ID, some methods~\cite{qin2023noisy, yang2024pedestrian, guo2023lidar, attribute2024} also use additional information such as textual description~\cite{qin2023noisy, yang2024pedestrian}, 3D morphological information~\cite{guo2023lidar}, or person attribute information~\cite{attribute2024} to augment the person re-identification performance. These methods utilizing external information are orthogonal to our method and can be potentially combined with our method.
In addition to self-attention, channel attention~\cite{hu2018squeeze, CBAM, MCA} adaptively re-weights feature channels based on their importance.

\subsection{Neural Architecture Search}
Existing Neural Architecture Search (NAS) methods can be grouped into two categories by optimization scheme, namely Differentiable NAS (DNAS) and Non-differentiable NAS. \hyphenation{Non-differentiable} NAS methods heavily rely on controllers based reinforcement learning \cite{zoph2016neural} or evolution algorithms \cite{real2019regularized} to discover better architecture. The DNAS methods, such as \cite{liu2018darts, shin2018differentiable}, search for optimal options for architecture in a handcrafted and finite option set. They transform the discrete network architecture space into a continuous space over which differentiable optimization is feasible and use gradient descent techniques to search the continuous space. For example, to search for the optimal filter numbers at different convolution layers, FBNetV2 \cite{wu2019fbnet, wan2020fbnetv2} models each option as a term with a Gumbel-Softmax mask. In our work, we also adopt a Gumbel-Softmax-based method to search for informative channels.

\subsection{Related Works about Information Bottleneck (IB)}
\label{sec:releated_IB}
The Information Bottleneck (IB) principle~\cite{NaftaliIB} is designed to extract latent representations from data that maintain essential information for a specific task while reducing redundant information in the input data. Deep VIB~\cite{AlemiFD017} applies the IB principle as an objective in training deep neural networks. Building upon the IB principle, \cite{lai2021information} introduces a spatial attention module that aims to decrease the mutual information between the modulated representation by attention and the input, simultaneously increasing the mutual information between this representation and the task label. Following that, \cite{zhou2022understanding} demonstrates that self-attention could be viewed as a recursive optimization process of the Information Bottleneck objective. Recent studies~\cite{AmjadG20, KawaguchiDJH23} have theoretically validated that reducing the information bottleneck can effectively reduce generalization errors of DNNs.
Although attention modules have been adopted in designing DNNs for person Re-ID, existing works on person Re-ID have not explored learning informative attention weights to enhance the representation learning capabilities of the attention modules. Furthermore, the IB principle has never been applied to enhance the image representations learned by DNNs for person Re-ID.
\section{Proposed Approach}
\label{sec::formulation}

DNNs based on attention methods, including self-attention and channel attention, have shown superior performance for person Re-ID due to their ability to dynamically learn discriminative features~\cite{chen2019abd, zhang2020relation, he2021transreid}. In this section, we first propose a novel distribution-free and efficient Information Bottleneck reduction framework, termed RIB. RIB aims to learn informative attention weights for self-attention and channel attention modules.
Under the general framework of RIB, Differentiable Channel Selection Attention, or DCS-Attention, is proposed to select more informative channels for learning more informative attention weights than vanilla self-attention. The DNNs with DCS-Attention enhanced by RIB are referred to as RIB-DCS.
RIB is also applied to existing channel attention modules to encourage the channel attention modules to assign higher attention weights to more informative feature channels.
The DNNs with channel attention enhanced by RIB are referred to as RIB-CA.
In Section~\ref{sec:method_RIB}, we present the formulation of RIB and a novel distribution-free and efficient variational upper bound for the IB loss.
In Section~\ref{sec:method_attention_application}, we introduce how RIB is applied to DNNs with DCS-Attention and existing channel attention modules, which leads to RIB-DCS and RIB-CA. In Section~\ref{sec:method_backbone}, we introduce how RIB-DCS and RIB-CA can be employed in fixed backbones designed for person Re-ID and learnable backbones that can be optimized by NAS methods.

\subsection{RIB: Reduction of Information Bottleneck Loss}
\label{sec:method_RIB}

Although attention methods have achieved remarkable success in person Re-ID, existing attention methods, including self-attention and channel attention, do not explicitly ensure that the attention weights are informative for predicting the identity of the person in the input image, and may consequently introduce information irrelevant to discriminative learning. To address this issue, we propose a novel information-theoretic feature learning method, termed Reduction of the Information Bottleneck loss (RIB), which explicitly encourages the attention modules to learn informative attention weights by reducing the Information Bottleneck (IB) loss of the associated DNNs. We first introduce the IB loss and the motivation for the reduction of the IB loss, and then introduce a new distribution-free and efficient variational upper bound for the IB loss, and the IB loss of the DNNs is reduced by optimizing such an upper bound.

RIB aims to reduce the IB loss, $\textup{IB}(F, X, Y) = I(F, X) - I(F ,Y)$, where $I(\cdot,\cdot)$ stands for the mutual information. $X$, $F$, and $Y$ denote the random variables representing the input feature, learned feature, and ground-truth training class label, respectively. Reduction of the IB loss ensures that the learned features are more correlated with the class labels (person identities) and less correlated with the input features. Existing self-attention and channel-attention modules, which do not explicitly learn informative attention weights, may introduce noisy information irrelevant to person Re-ID by attending to less informative feature tokens and feature channels that potentially correspond to non-human regions of the input image.
By reducing the IB loss, RIB explicitly encourages more informative feature tokens in self-attention and feature channels in channel attention to receive higher attention weights. This is achieved by increasing the correlation of the learned features with class labels while decreasing their correlation with the input images containing backgrounds or occlusions, ultimately leading to more discriminative features.
Given the training data $\set{X_i,y_i}_{i=1}^n$ where $X_i$ is the $i$-th input training feature and $y_i$ is the corresponding class label or the identity of the person, we first specify how to compute the IB loss based on the training data. Let $\textup{IB}(\cW) = I(F(\cW),X)
-I(F(\cW),Y)$ denote the IB loss where $\cW$ denotes the weights of the DNN.
$X$ takes values in $\set{X_i}_{i=1}^n$.
$F(\cW)$ takes values in $\set{F_i(\cW)}_{i=1}^n$ with $F_i(\cW)$ being the $i$-th learned feature by the network. $F(\cW)$ and $F_i(\cW)$ are also abbreviated as $F$ and $F_i$ for simplicity. $Y$ takes values in $\set{y_i}_{i=1}^n$. We have the class centroids $\set{\cC_a}_{a=1}^C$ and $\set{\cC_b}_{b=1}^C$ for the learned features $\set{F_i}_{i=1}^n$ and the input features $\set{X_i}_{i=1}^n$, respectively, where each $\cC_a$ is the average of the learned features in class $a$, each $\cC_b$ is the average of the input features in class $b$, and $C$ is the number of classes.  Then we define the probability that $F_i$ belongs to class $\cC_a$ as $\Prob{F \in a} = \frac 1n \sum\limits_{i=1}^n  \phi(F_i,a)$ with
$\phi(F_i,a) = \frac{\exp\left(-\ltwonorm{F_i - \cC_a}^2\right)}{\sum_{a=1}^{C}\exp\left(-\ltwonorm{F_i - \cC_a}^2\right)}$. Similarly, we define the probability that $X_i$ belongs to class $\cC_b$
as $\Prob{X \in b}
= \frac 1n \sum\limits_{i=1}^n  \phi(X_i,b)$.
We can compute the mutual information $I(F, X)$ and $I(F, Y)$ by
\noindent\resizebox{1\columnwidth}{!}{
    \begin{minipage}{1\columnwidth}
        \bals
        I(F, X) &= \sum\limits_{a=1}^C \sum\limits_{b=1}^C
        \Prob{F \in a, X \in b} \log{\frac{\Prob{F \in a, X \in b}}
        {\Prob{F \in a}\Prob{X \in b}}},
        \eals
        \vspace{1mm}
    \end{minipage}
}
\noindent\resizebox{1\columnwidth}{!}{
    \begin{minipage}{1\columnwidth}
        \bals
        I(F, Y) &= \sum\limits_{a=1}^C \sum\limits_{y=1}^C
        \Prob{F \in a, Y = y} \log{\frac{\Prob{F \in a, Y = y}}
        {\Prob{F \in a}\Prob{Y = y}}},
        \eals
        \vspace{1mm}
    \end{minipage}
}
and then compute the IB loss $\textup{IB}(\cW)$.
Given a variational distribution $Q(F \in a| Y=y)$ for $y \in \set{1,\ldots C}$ and $a \in \set{1,\ldots C}$, the following theorem gives a variational upper bound, $\textup{IBB}(\cW)$, for the IB loss $\textup{IB}(\cW)$. $\textup{IBB}(\cW)$ is also abbreviated as $\textup{IBB}$ in the following text.
\begin{theorem}\label{theorem:IB-upper-bound}
Let $\Prob{X \in y}  = \sum_{i=1}^n \indict{y_i = y}/n \defeq p_y$ be the prior probability for every $y \in [C]$, we have
\bal\label{eq:IB-upper-bound}
\textup{IB}(\cW) \le \textup{IBB}(\cW) ,
\eal
where
\bsals
&\textup{IBB}(\cW) \nonumber \\
&\defeq  \frac 1{n} \sum\limits_{i=1}^n \sum\limits_{a=1}^C \sum\limits_{y=1}^C
\indict{y_i = y} \phi(F_i, a) \log \pth{ \frac{\indict{y_i = y}}{p_y Q(F \in a| Y=y)} }.
\esals
\end{theorem}

The proof of this theorem follows by applying Lemma~\ref{lemma:I-X-tildeX-upper-bound} and Lemma~\ref{lemma:I-tildeX-Y-lower-bound} in Section~\ref{sec:proofs} of the supplementary.
We remark that $\textup{IBB}(\cW)$ is ready to be optimized by standard SGD algorithms because it is separable and expressed as the summation of losses on individual training points.
Algorithm~\ref{Algorithm-IBB} describes the training process of a RIB model for person Re-ID, where $\textup{IBB}(\cW)$ is a regularization term in the training loss.
The following functions are needed for minibatch-based training with SGD, with the subscript $j$ indicating the corresponding loss on the $j$-th batch $\cB_j$, and the superscript $(t)$ indicating the corresponding loss at the $t$-th epoch:
\bsals
&\textup{IBB}^{(t)}_{j}(\cW) \nonumber \\
&=  \frac 1{\abth{\cB_j}} \sum\limits_{i\in \cB_j} \sum\limits_{a=1}^C \sum\limits_{y=1}^C
\indict{y_i = y} \phi(F_i, a) \log \pth{ \frac{\indict{y_i = y}}{p_y Q^{(t)}(F \in a| Y=y)} },
\esals
where
\bal
    \mathcal{L}^{(t)}_{\text{train},j}(\cW) &= \text{CE}^{(t)}_{j} + \textup{Triplet}_j^{(t)}+ \eta \textup{IBB}^{(t)}_{j}(\cW),\label{eq:train_loss}
\eal
\vspace{-5mm}
\bals
    \text{CE}^{(t)}_{j} &=  \frac{1}{\abth{\cB_j}}\sum_{i \in \cB_j}H(F_i(\cW), Y_i),  \nonumber
\eals
\vspace{-5mm}
\bals
     &\text{Triplet}^{(t)}_{j} =  \frac{1}{\abth{\cB_j}}\sum_{i\in \cB_j}\max\pth{\|F_i - F_i^{\textup{pos}}\|_2- \|F_i -F_i^{\textup{neg}}\|_2,0}. \nonumber
\eals
Here $Q^{(t)}(F \in a| Y=y)$ is the variational conditional probability that a feature $F$ belongs to cluster $a$ given class label $y$ at the $t$-th epoch, which is computed efficiently by Algorithm~\ref{alg:Q_computation} in the supplementary.
$\text{CE}^{(t)}_{j}$ is the cross-entropy loss on batch $\cB_j$ at epoch $t$.  $H(\cdot,\cdot)$ is the cross-entropy function. $\eta$ is the balance factor for IBB. The values of $\eta$ on different datasets will be decided by performing cross-validation. $\text{Triplet}^{(t)}_{j}$ is the triplet loss~\cite{chen2017beyond} on batch $\cB_j$ at epoch $t$. $\|\cdot\|_2$ denotes the Euclidean norm. $F_i^{\textup{pos}}$ is the feature of another image for the same person as $F_i$.  $F_i^{\textup{neg}}$ is the feature of another image for a different person than $F_i$.
Let $T_0$ denote the computational complexity of a forward and backward pass of the neural network with respect to each training sample. The overall computational complexity for calculating the proposed IBB regularization term $\textup{IBB}(\cW)$ is $\Theta(n C T_0 + C^2) = \Theta(nC T_0)$ with
$n \ge C$. In contrast, computing the upper bound for the mutual information required for calculating the upper bound for the IB loss proposed in CLUB~\cite{CLUB} requires a substantially higher computational complexity of $\Theta(n^2 T_0)$. Note that $\Theta(n^2 T_0)$ corresponds exclusively to the upper bound for the mutual information $I(F,X)$, while CLUB additionally requires the computation of the lower bound of the mutual information $I(F,Y)$.
Details on the complexity analysis of CLUB and IBB are presented in Section~\ref{sec:complexity} of the supplementary.
Moreover, Table~\ref{tab:compare-IBB-VIB-APIB} in Section~\ref{sec:ablation_IBB} demonstrates that IBB achieves substantially better performance than CLUB and other competing methods that rely on the unrealistic Gaussian distribution assumption of hidden features~\cite{VIB-DaiZGW18, VIB-SrivastavaDGAA21} or on the approximation to the IB loss~\cite{IB-lasso-APIB}.

\begin{algorithm}[!htb]
\caption{Training Algorithm of the RIB models}\label{Algorithm-IBB}
{
\begin{algorithmic}[1]
\small
\REQUIRE Training data $\set{X_i,y_i}_{i=1}^n$, mini-batches of the training data $\set{B_j}_{j=1}^{J}$ the epoch number $t_{\text{train}}$, and learning rate $\alpha$.
\ENSURE The trained weights $\cW$ of the network.
\STATE Initialize the weights of the network $\cW$ through random initialization.
\STATE Initialize $Q^{(0)}(F|Y) = \tfrac{1}{C}\mathbf{1}^{C \times C}$.
\STATE Compute $\set{\cC_b}_{b=1}^C$ by $\cC_b = \frac{\sum_{i=1}^{n} \indict{y_i = b} X_i}{\sum_{i=1}^{n} \indict{y_i = b}}$.
\FOR{$t\leftarrow 1$ to $t_{\text{train}}$}
\FOR{$j \leftarrow 1$ to $J$}
\STATE Update $\cW \leftarrow \cW - \alpha \nabla_{\cW} \mathcal{L}^{(t)}_{\text{train},j}(\cW) $ using mini-batch gradient descent on batch $\cB_j$.
\ENDFOR
\STATE Compute $Q^{(t)}(F| Y)$ by Algorithm~\ref{alg:Q_computation} in Section \ref{sec:complexity} of the supplementary.
\STATE Compute  $\set{\cC_a}_{a=1}^C$ by $\cC_a = \frac{\sum_{i=1}^{n} \indict{y_i = a} F_i(\cW)}{\sum_{i=1}^{n} \indict{y_i = a}}$.
\ENDFOR
\STATE \textbf{return} The trained weights $\cW$ of the network.
\end{algorithmic}
}

\end{algorithm}

\vspace{-3mm}
\subsection{Application of RIB to Self-Attention and Channel Attention}
\label{sec:method_attention_application}
By adding IBB to the training loss in Equation~(\ref{eq:train_loss}), RIB can be applied to DNNs with self-attention or channel attention modules to promote learning informative attention weights, leading to RIB-DCS and RIB-CA, respectively.
To apply RIB to self-attention, we propose learning channel selection masks over the features used to compute the attention weight matrix. As a result, RIB explicitly encourages the use of more informative channels to learn more informative attention weights than vanilla self-attention modules. The channel selection masks are learned in a differentiable manner with a novel Differentiable Channel Selection Attention module, or DCS-Attention, so that they can be jointly optimized with the network parameters, leading to competitive DNNs termed RIB-DCS.
RIB can also be applied to DNNs with existing channel attention modules by adding IBB to the training loss, leading to RIB-CA.

\begin{figure}[!htb]
\vspace{-2mm}
\begin{center}
 \includegraphics[width=0.8\columnwidth]{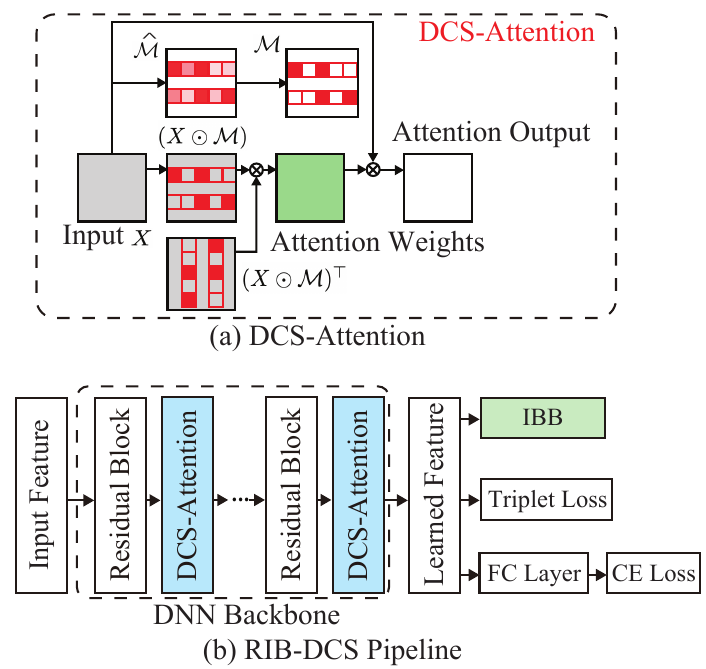}
\end{center}
\vspace{-4mm}
\caption{Figure (a) illustrates the framework of the DCS-Attention module. Figure (b) illustrates the pipeline for the RIB-DCS model. The DCS-Attention modules are inserted after each residual block. In a RIB-CA model, the DCS-Attention modules are replaced with existing channel attention modules, such as SE~\cite{hu2018squeeze}, CBAM~\cite{CBAM}, and MCA~\cite{MCA}. IBB is added as a regularization term in addition to the triplet loss and the cross-entropy (CE) loss.
}
\vspace{-6mm}
\label{fig:pipeline}
\end{figure}

\noindent\textbf{RIB-DCS: RIB with Differentiable Channel Selection for Attention Weight.}
To introduce how channel selection is performed on features for computing attention weights, we first present the preliminaries of self-attention modules.
Let $X$ be an input feature of shape $ N \times D$, where $N$ is the number of feature tokens and $D$ is the number of feature dimensions. If the input is a three-dimensional feature map in CNNs of the shape $H\times W\times D$ where $H$ and $W$ are the height and width of the feature map, it is first reshaped into a tensor of size $N\times D$, where $N=HW$. The vanilla self-attention module applies three linear layers to $X$ to obtain the query ($Q$), key ($K$), and value ($V$). The attention output is computed as a weighted aggregation of the feature tokens in the value $V$ by
${\rm Output}(Q,K,V) = {\rm softmax}(QK^{\top})V$.
The self-attention modules used for the computer vision tasks, such as the non-local self-attention \cite{wang2018non}, are formulated as
${\rm Output} = f (g_q({X}),g_k({X})) g_v({X})$,
where $g_q(\cdot)$, $g_k(\cdot)$ and $g_v(\cdot)$ are transformations applied on the input $X$, such as a convolution layer or a linear layer. $f(\cdot, \cdot)$ is the function that measures the feature correlation, such as the dot-product or cosine similarity.
Recent work \cite{li2020neural} demonstrates that $g_q(\cdot)$, $g_k(\cdot)$, and $g_v(\cdot)$ can be removed due to the strong capability of function approximation of DNNs for computer vision tasks. In this work, we use the self-attention module from \cite{li2020neural}, which is formulated as
${\rm Output} =\frac{1}{C(X)}\sigma(XX^{\top})X$. $C(X)$ is a normalization factor, and $C(X)=1$ when $\sigma(\cdot)$ is the Softmax operation.

Under the general framework of RIB, we aim to select informative channels to compute the attention weights. To this end, we propose DCS-Attention, which performs differentiable channel selection in the features used for the computation of the attention weights.
The channel selection is performed by maintaining a binary decision mask $\cM \in \{0,1\}^{N \times D}$, where its $(i,d)$-th element $\cM_{id}=1$ indicates the $d$-th channel of $X_i$ is selected. $X_i$ is the $i$-th row of $X$. As a result, the attention weights $A$ are computed by $A = \sigma\pth{(X \odot \cM)(X\odot \cM)^{\top}}$,
where $\odot$ indicates the elementwise product. To optimize the discrete binary decision mask with gradient descent, we adopt a simplified binary Gumbel-Softmax~\cite{verelst2020dynamic} to relax $\cM\in \{0,1\}^{N\times D}$ into its approximation in the continuous domain $\hat \cM\in(0,1)^{N \times D}$. The approximated soft decision mask can be computed by $\hat \cM_{id} = \phi \Bigl( \frac{ \theta_{id} + \epsilon_{id}^{(1)} - \epsilon_{id}^{(2)} }{\tau} \Bigr)$,
where $\hat \cM_{id}$ is the $(i,d)$-th element of $\hat \cM$. $\epsilon_{id}^{(1)}$ and $\epsilon_{id}^{(2)}$ are Gumbel noise terms for the approximation. $\tau$ is the temperature, and $\phi(\cdot)$ is the sigmoid function. $\theta\in \mathbb{R}^{N\times D}$ is the sampling parameter.
In this work, we obtain $\theta$ by applying a linear layer on $X$ so that the soft decision mask is dependent on the input features of that layer.
Inspired by the straight-through estimator \cite{verelst2020dynamic,Bengio2013EstimatingComputation}, we directly set $ \cM=\hat \cM$ in the backward pass. In the forward pass, the binary decision mask is computed by $ \cM_{id} = 1$ if $\hat \cM_{id} > 0.5$, and $ \cM_{id} = 0$  if $\hat \cM_{id} \leq 0.5$.
During inference, the Gumbel noise $\epsilon_{id}^{(1)}$ and $\epsilon_{id}^{(2)}$ are set to $0$. In this way, the informative channels for the attention weights computation are selected in a differentiable manner.
To encourage the channel selection masks to select more informative channels in the features for computing the attention weights, the training of the models with DCS-Attention is supervised by the loss function defined in Equation~(\ref{eq:train_loss}), leading to RIB-DCS.
Fig.~\ref{fig:pipeline}(a) illustrates the framework of the DCS-Attention module.
Fig.~\ref{fig:pipeline}(b) illustrates the pipeline of the RIB-DCS model.

\noindent\textbf{RIB-CA: RIB with Existing Channel Attention Methods. }
Similar to existing self-attention modules that learn attention weights to quantify the importance of feature tokens, channel attention modules learn attention weights for reweighting different feature channels.
However, existing channel attention modules, such as Squeeze-and-Excitation (SE)~\cite{hu2018squeeze}, Convolutional Block Attention Module (CBAM)~\cite{CBAM}, and Moment Channel Attention (MCA)~\cite{MCA}, learn attention weights using activation statistics or heuristic gating functions. As a result, existing channel attention modules do not explicitly ensure that the channels with higher attention weights are more informative for the discriminative task of person Re-ID.
To address this issue, RIB is applied to DNNs equipped with existing channel attention modules, including SE, CBAM, and MCA, in this paper, which leads to RIB-CA.
Similar to RIB-DCS, the training of RIB-CA is performed by minimizing the loss function defined in Equation~(\ref{eq:train_loss}).

\subsection{RIB with Fixed Backbone and Learnable Backbone}
\label{sec:method_backbone}

To evaluate the performance of the proposed RIB, we incorporate RIB-DCS and RIB-CA into both Fixed Backbones (FB), which lead to RIB-DCS-FB and RIB-CA-FB, and learnable backbones with the Differentiable Neural Architecture Search (DNAS) algorithm, which lead to RIB-DCS-DNAS and RIB-CA-DNAS. The fixed neural network backbones used in this work include MobileNetV2~\cite{sandler2018mobilenetv2}, OSNet~\cite{zhou2019omni}, ResNet50~\cite{resnet}, HRNet~\cite{dou2022reliability}, and TransReID~\cite{wang2022pose}. We use FBNetV2~\cite{wan2020fbnetv2} as the learnable neural network backbone.
For RIB-DCS-DNAS models, we first apply the standard DNAS algorithm in FBNetV2 to search for the architectures of the neural network backbones. In the search phase of the RIB-DCS-DNAS models, the weights of the neural network and a set of architecture parameters characterizing the selection of different architecture options in the supernet are jointly optimized.
Let $\cW$ denote the weights of the neural network, and $\cV$ denote the architecture parameters in the supernet. The search loss of RIB-DCS-DNAS and RIB-CA-DNAS models at the $t$-th epoch on the $j$-th batch $\cB_j$ is
\bal\label{eq:search_loss}
\cL^{(t)}_{\textup{search},j}(\cW, \cV) = \cL^{(t)}_{\textup{train},j}(\cW, \cV) + \cL_{\textup{latency}}(\cV),
\eal
where $\cL_{\textup{latency}}(\cV)$ is the latency cost of the network. We use the same latency cost as in FBNetV2~\cite{wan2020fbnetv2}. $\cL^{(t)}_{\textup{train},j}(\cW, \cV)$ is the training loss in Equation~(\ref{eq:train_loss}).
Once the search process is finished, the RIB-DCS-DNAS and RIB-CA-DNAS models with the searched architectures will be retrained. For both the search and the retraining of the RIB-DCS-DNAS and RIB-CA-DNAS models, minibatch-based gradient descent algorithms are used to optimize the network parameters by minimizing the joint training loss $\mathcal{L}^{(t)}_{\text{train},j}(\cW)$ from Equation~(\ref{eq:train_loss}) for the data in batch $\cB_j$ at the $t$-th epoch.

\section{Experiments}
\label{sec::experiments}
In this section, we evaluate the performance of the RIB models for the person Re-ID task on public person Re-ID benchmarks. In Section~\ref{sec::dataset-metric}, we present the details of the person Re-ID benchmarks and evaluation metrics. In Section~\ref{sec:implementation_details}, we present the implementation details and training settings of the RIB-DCS-FB and RIB-DCS-DNAS models. In Section~\ref{sec::results-sw-DCS}, we present the evaluation results of RIB-DCS-FB and RIB-DCS-DNAS models.
In Section~\ref{sec:ChannelAttn_Results}, we present the evaluation results of RIB-CA-FB and RIB-CA-DNAS models.
In Section~\ref{sec:ablation_IB_loss}, we perform ablation studies on the effectiveness of IBB.
In Section~\ref{sec:ablation_IBB}, we compare both the efficiency and the performance between the proposed IBB and existing methods for reducing the IB loss.
In Section~\ref{sec:occluded_duke}, we evaluate RIB models for occluded person Re-ID.
The training time evaluation of RIB models is presented in Section \ref{sec:training_time}. The search time evaluation of RIB-DCS-DNAS and RIB-CA-DNAS models is conducted in Section \ref{sec:search_time}.
Additional experimental results are presented in Section~\ref{sec:add_exp_sup} of the supplementary.
Details of the evaluation datasets are presented in Section~\ref{sec:dataset_sup}.
The structure of the supernet of FBNetV2, FBNetV2-Large, and FBNetV2-XLarge is detailed in Section~\ref{sec:FBNet_ARCH}.
We compare RIB-DCS with an existing attention masking method, MODA~\cite{MODA}, in Section~\ref{sec:moda}.
The evaluation results on the CUHK03 dataset are presented in Section~\ref{sec:cuhk03}.
In Section~\ref{sec:cross_dataset}, we further evaluate the cross-domain person Re-ID performance of RIB by integrating it into the domain transfer method, CaCL~\cite{LeeL0SYH23}.
In Section~\ref{sec:SSL}, we study the effectiveness of RIB in self-supervised person re-identification frameworks.
In Section~\ref{sec:ablation_dcs}, we perform ablation studies on the DCS-Attention modules in the RIB-DCS-FB and RIB-DCS-DNAS models.
In Section~\ref{sec:grad_cam}, we include more GradCAM visualization analysis.
t-SNE visualization analysis on features learned by RIB-DCS-FB (TransReID) and the baseline model, TransReID, is performed in Section~\ref{sec:tsne} of the supplementary.
In Section~\ref{sec:loss_figure} of the supplementary, we illustrate the plots of the IBB and the test loss of RIB-DCS-FB models during training.
In Section~\ref{sec:sup_hard-channel-selection} of the supplementary, we demonstrate the advantages of the hard channel selection in DCS-Attention over existing channel attention methods.
\subsection{Datasets and Evaluation Metrics}
\label{sec::dataset-metric}

We evaluate the RIB models on four public person Re-ID datasets, which are Market-1501 \cite{market}, DukeMTMC \cite{duke}, MSMT17 \cite{msmt}, and CUHK03 \cite{cuhk03}.
Standard Re-ID metrics Rank-1 accuracy and the mean Average Precision (mAP) are used. Re-ranking \cite{zhong2017re-ranking} and multi-query fusion \cite{market} are not used for the fairness of comparison. Statistics of the evaluation datasets are presented in Section~\ref{sec:dataset_sup} of the supplementary.

\begin{table*}[!t]
\caption{Performance of RIB-DCS-FB models with comparisons to state-of-the-art Re-ID models. Smaller-sized backbone enjoys less FLOPs while larger-sized backbones, such as TransReID, exhibit SOTA performance. The improvements of the RIB models over their corresponding baselines are reported in parentheses after each metric value.}
\vspace{-3mm}
\label{SOTA}
\begin{center}
\resizebox{1\textwidth}{!}{
\begin{tabular}{lcccccccccc}
\toprule
\multirow{2}{*}{Methods} & \multirow{2}{*}{Backbones} & \multirow{2}{*}{Input Size}& \multirow{2}{*}{Params(M)}& \multirow{2}{*}{FLOPs(G)} & \multicolumn{2}{c}{Market1501} & \multicolumn{2}{c}{DukeMTMC} & \multicolumn{2}{c}{MSMT17}\\ \cmidrule{6-11}
 &&& & & mAP& Rank-1& mAP & Rank-1& mAP & Rank-1\\ \midrule
\multicolumn{11}{c}{Trained from scratch} \\ \midrule
HACNN \cite{li2018harmonious} & Inception& 160$\times$64 & 4.5 & 0.55& 79.9 & 92.3& 63.8& 80.5& 41.5& 70.6\\
OSNet \cite{zhou2019omni} & OSNet& 256$\times$128 & 2.2 & 0.98& 81.0 & 93.6& 68.6& 84.7& 43.3& 71.0\\
Auto-ReID \cite{chen2019abd}& ResNet50 & 384$\times$128 & 13.1& 2.05& 74.6 & 90.7& 60.8& 78.6& 32.5&60.5 \\
RGA \cite{zhang2020relation}& MobileNetV2& 256$\times$128 & 5.13& 2.63& 81.5 & 92.9& 70.5& 83.7& 36.8& 62.1\\
MobileNetV2~\cite{sandler2018mobilenetv2} & MobileNetV2& 256$\times$128 & 2.22& 0.380 & 78.9 & 92.0&68.5 &82.7&34.6&61.2 \\
MobileNetV2-200~\cite{sandler2018mobilenetv2} & MobileNetV2-200 & 256$\times$128 & 5.04 & 0.890
& 81.8 & 93.4
& 71.8 & 84.5
& 43.2 & 68.8 \\
\midrule
\textbf{RIB-DCS-FB (OSNet)} & OSNet & 256$\times$128 & 2.33 & 1.03
& {83.7} ($\uparrow$\,2.7)
& {94.2} ($\uparrow$\,0.6)
& {74.0} ($\uparrow$\,5.4)
& {85.8} ($\uparrow$\,1.1)
& {46.5} ($\uparrow$\,3.2)
& {73.2} ($\uparrow$\,2.2) \\
\textbf{RIB-DCS-FB (MobileNetV2)} & MobileNetV2 & 256$\times$128 & 2.31 & 0.406
& {82.5} ($\uparrow$\,3.6)
& {93.0} ($\uparrow$\,1.0)
& {74.5} ($\uparrow$\,6.0)
& {86.0} ($\uparrow$\,3.3)
& {45.3} ($\uparrow$\,10.7)
& {72.9} ($\uparrow$\,11.7) \\
\textbf{RIB-DCS-FB (MobileNetV2-200)} & MobileNetV2-200 & 256$\times$128 & 5.32 & 0.927
& {84.5} ($\uparrow$\,2.7)
& {94.5} ($\uparrow$\,1.1)
& {74.9} ($\uparrow$\,3.1)
& {87.0} ($\uparrow$\,2.5)
& {47.0} ($\uparrow$\,3.8)
& {73.5} ($\uparrow$\,4.7) \\
\midrule
\multicolumn{11}{c}{Pre-trained on ImageNet}\\ \midrule
MobileNetV2~\cite{sandler2018mobilenetv2} & MobileNetV2 & 256$\times$128 & 2.22 & 0.380
& 84.5 & 93.9
& 72.3 & 85.5
& 47.6 & 74.9 \\
MobileNetV2-200~\cite{sandler2018mobilenetv2} & MobileNetV2-200 & 256$\times$128 & 5.04 & 0.890
& 85.0 & 94.5
& 75.2 & 87.4
& 50.0 & 76.3 \\
ResNet50~\cite{resnet} & ResNet50 & 256$\times$128 & 23.5 & 6.02
& 86.8 & 95.5
& 78.1 & 88.5
& 57.3 & 81.6 \\
AANet \cite{tay2019aanet} & ResNet50 & 256$\times$128 &27.1 & 7.36& 85.3 & 94.7& 75.3& 84.0& 58.6&81.5 \\
CAMA \cite{yang2019cama}& ResNet50 & 256$\times$128 & 28.4 & 7.85& 84.5 & 94.7& 72.9& 85.8& 57.4 & 79.5\\
BAT-Net \cite{fang2019bilinear} & ResNet50 & 256$\times$128 & 26.8& 7.54& 87.4 & 95.1& 77.3& 87.7& 58.8& 81.0\\
ABD-Net \cite{chen2019abd}& ResNet50 & 384$\times$128 & 69.2 & 14.1& 88.3& 95.6& 78.6 & 89.0& {60.8} &{82.3} \\
Auto-ReID \cite{autoreid} & ResNet50 & 384$\times$128 & 13.1& 2.05& 85.1 & 94.5&77.6 &88.8 & 52.5& 78.2\\
OSNet \cite{zhou2019omni}& OSNet& 256$\times$128 & 2.2 & 0.98& 84.9 & 94.8& 73.5& 88.6& 52.9& 78.7\\
RGA \cite{zhang2020relation}& ResNet50 & 256$\times$128 & 28.3& 7.10& 87.5 & 96.0&77.4 &88.8 & 57.5& 80.3\\
PAT \cite{li2021diverse}& ResNet50 & 256$\times$128 & 27.8& 7.95& 88.0 & 95.4& 78.2 &88.8 & 58.9& 81.0\\
AutoLoss-GMS \cite{gu2022autoloss}& ResNet50 & 256$\times$128 & 23.9& 6.89& 87.0 & 94.7&78.5 &89.2 & 55.1& 79.5\\
UAL \cite{dou2022reliability}& ResNet50 & 256$\times$128 & 24.0& 7.5& 87.0 & 95.2&80.0 &89.5 & 56.5 & 80.0\\
UAL \cite{dou2022reliability}& HRNet & 256$\times$128 & 30.5 & 9.0& 89.5 & 95.7&82.5 &90.8 & 65.3& 84.7\\
TransReID \cite{he2021transreid}& ViT-B/16& 256$\times$128 & 100&19.3& 88.9 & 95.2& 82.0 & 90.7&67.4& 85.3\\
BPBreID \cite{somers2023body}& ResNet50& 256$\times$128 & 24.4 & 6.9 & 87.0 & 95.1&78.3 &89.6 &57.0& 80.9\\
PFD \cite{wang2022pose}& TransReID& 256$\times$128 & 100& 19.3& 88.7 & 95.5& \underline{83.2} & \underline{91.2}&68.5& 85.8 \\
PHA \cite{zhang2023pha}& TransReID& 256$\times$128 & 100& 19.3& \underline{90.3} & 96.1&83.0&91.0&\underline{68.9}&\underline{86.1} \\
\midrule

\textbf{RIB-DCS-FB (MobileNetV2)} & MobileNetV2 & 256$\times$128 & 2.31 & 0.406
& 86.4 ($\uparrow$\,1.9) & 95.1 ($\uparrow$\,1.2)
& 75.6 ($\uparrow$\,3.3) & 86.9 ($\uparrow$\,1.4)
& 53.1 ($\uparrow$\,5.5) & 78.5 ($\uparrow$\,3.6) \\

\textbf{RIB-DCS-FB (MobileNetV2-200)} & MobileNetV2-200 & 256$\times$128 & 5.32 & 0.927
& 87.5 ($\uparrow$\,2.5) & 95.6 ($\uparrow$\,1.1)
& 78.6 ($\uparrow$\,3.4) & 89.1 ($\uparrow$\,1.7)
& 54.8 ($\uparrow$\,4.8) & 78.9 ($\uparrow$\,2.6) \\

\textbf{RIB-DCS-FB (ResNet50)} & ResNet50 & 256$\times$128 & 23.9 & 6.24
& 88.7 ($\uparrow$\,1.9) & 96.0 ($\uparrow$\,0.5)
& 76.8 ($\downarrow$\,1.3) & 88.9 ($\uparrow$\,0.4)
& 58.9 ($\uparrow$\,1.6) & 82.7 ($\uparrow$\,1.1) \\
\textbf{RIB-DCS-FB (OSNet)} & OSNet & 256$\times$128 & 2.33 & 1.03
& 86.8 ($\uparrow$\,1.9) & 94.9 ($\uparrow$\,0.1)
& 75.8 ($\uparrow$\,2.3) & 89.0 ($\uparrow$\,0.4)
& 55.1 ($\uparrow$\,2.2) & 79.8 ($\uparrow$\,1.1) \\
\textbf{RIB-DCS-FB (HRNet)} & HRNet & 256$\times$128 & 32.1 & 9.9
& 90.1 ($\uparrow$\,0.6) & 96.2 ($\uparrow$\,0.5)
& 82.3 ($\downarrow$\,0.2) & 91.1 ($\uparrow$\,0.3)
& 66.9 ($\uparrow$\,1.6) & 85.1 ($\uparrow$\,0.4) \\

\textbf{RIB-DCS-FB (TransReID)} & TransReID & 256$\times$128 & 101.6 & 19.7
& \textbf{91.3} ($\uparrow$\,2.4) & \textbf{97.0} ($\uparrow$\,1.8)
& \textbf{84.3} ($\uparrow$\,2.3) & \textbf{92.0} ($\uparrow$\,1.3)
& \textbf{70.2} ($\uparrow$\,2.8) & \textbf{87.1} ($\uparrow$\,1.8) \\
\bottomrule
\end{tabular}
}
\end{center}
\vspace{-2mm}

\end{table*}

\begin{table*}[!htbp]
\centering
\scriptsize
\caption{Performance of RIB-DCS-DNAS models on Market1501, DukeMTMC, and MSMT17. The improvements of the RIB models over their corresponding baselines are reported in parentheses after each metric value.}
\vspace{-2mm}
\resizebox{0.9\linewidth}{!}{
\begin{tabular}{lcccccccc}
\toprule
\multirow{2}{*}{Method} & \multirow{2}{*}{Params(M)} & \multirow{2}{*}{FLOPs(G)} & \multicolumn{2}{c}{Market1501} & \multicolumn{2}{c}{DukeMTMC} & \multicolumn{2}{c}{MSMT17} \\
\cmidrule{4-9}
 &&& mAP & Rank-1 & mAP & Rank-1 & mAP & Rank-1 \\
 \midrule
Auto-ReID \cite{autoreid} & 13.1 & 2.05 & 85.1 & 94.5 & 77.6 & 88.8 & 52.5 & 78.2 \\
RGA \cite{zhang2020relation} & 28.3 & 7.3 & 87.5 & 96.0 & 77.4 & 88.8 & 57.5 & 80.3 \\
UAL \cite{dou2022reliability} & 24.0 & 7.5 & 87.0 & 95.2 & 80.0 & 89.5 & 56.5 & 80.0 \\
BPBreID \cite{somers2023body} & 24.4 & 6.9 & 87.0 & 95.1 & 78.3 & 89.6 & 57.0 & 80.9 \\
ABD-Net \cite{chen2019abd} & 69.2 & 14.1 & 88.3 & 95.6 & 78.6 & 89.0 & 60.8 & 82.3 \\
\midrule
FBNetV2 & 5.9 & 0.572
& 84.6 & 94.1 & 75.9 & 87.4 & 58.1 & 80.5 \\

\textbf{RIB-DCS-DNAS (FBNetV2)} & 6.3 & 0.624
& {85.8} ($\uparrow$\,1.2)
& {95.3} ($\uparrow$\,1.2)
& {77.1} ($\uparrow$\,1.2)
& {88.6} ($\uparrow$\,1.2)
& {59.5} ($\uparrow$\,1.4)
& {81.9} ($\uparrow$\,1.4) \\
\midrule

FBNetV2-Large & 13.2 & 1.235
& 87.2 & 95.0 & 77.8 & 88.9 & 59.9 & 82.5 \\

\textbf{RIB-DCS-DNAS (FBNetV2-Large)} & 13.7 & 1.302
& {89.1} ($\uparrow$\,1.9)
& {95.8} ($\uparrow$\,0.8)
& {79.2} ($\uparrow$\,1.4)
& {90.2} ($\uparrow$\,1.3)
& {61.0} ($\uparrow$\,1.1)
& {83.0} ($\uparrow$\,0.5) \\
\midrule

FBNetV2-XLarge & 24.5 & 1.798
& 87.9 & 95.3 & 79.5 & 90.1 & 61.3 & 82.9 \\

\textbf{RIB-DCS-DNAS (FBNetV2-XLarge)} & 25.2 & 1.898
& \textbf{89.8} ($\uparrow$\,1.9)
& \textbf{96.4} ($\uparrow$\,1.1)
& \textbf{81.0} ($\uparrow$\,1.5)
& \textbf{91.4} ($\uparrow$\,1.3)
& \textbf{62.9} ($\uparrow$\,1.6)
& \textbf{83.6} ($\uparrow$\,0.7) \\
\bottomrule
\label{SOTA-DNAS}
\end{tabular}
}
\vspace{-3mm}
\end{table*}

\subsection{Implementation Details}
\label{sec:implementation_details}


\textbf{Training Settings of RIB Models with Fixed Backbones, including RIB-DCS-FB and RIB-CA-FB.} Following the settings in existing works~\cite{he2021transreid, zhang2020relation} for person Re-ID, random cropping, horizontal flipping, and random erasing are used to augment the training data. Label smoothing \cite{szegedy2016rethinking} is used to augment the labels of the cross-entropy loss. All models are trained with momentum SGD for $600$ epochs. The batch size is set to $256$. The momentum for SGD is set to $0.9$. The initial learning rate is set to $0.035$, and we decay the learning rate by $10$ at the $300$-th and $500$-th epochs. We set the weight decay of SGD to $0.0005$.
To stabilize the training of DCS-Attention modules, we include a warm-up process that disables channel selection for the first $10$ epochs.

\noindent\textbf{Training Settings of RIB Models with Learnable Backbones.} The search phase of RIB-DCS-DNAS and RIB-CA-DNAS models take $300$ epochs. In each epoch, the network weights are trained with $80\%$ of the training samples by SGD. The architecture parameters are trained with the remaining $20\%$ using Adam. The initial learning rate for optimizing architecture parameters is set to $0.03$, and a cosine learning rate schedule is used. The initial learning rate for network parameters is set to $0.035$, and it is decayed by $10$ at the $150$-th and $240$-th epochs.
After the search phase, we first pre-train the model with the searched architectures on ImageNet. Then we fine-tune the model on the Re-ID datasets. Adam is used to fine-tune the network. The initial learning rate is set to $0.00035$.
All models are fine-tuned for $600$ epochs. The learning rate is decayed by $10$ after $300$ and $500$ epochs.
In our work, we build RIB-DCS-DNAS models based on FBNetV2 \cite{wan2020fbnetv2}.
The settings for tuning $\eta$, which is the balance factor for the IBB, and the structure of the supernet of FBNetV2, FBNetV2-Large, and FBNetV2-XLarge are detailed in Section~\ref{sec:FBNet_ARCH} of the supplementary.

\subsection{RIB-DCS for Person Re-ID}
\label{sec::results-sw-DCS}
\noindent\textbf{Results for RIB-DCS-FB.}
In this section, we apply RIB-DCS to multiple fixed CNN backbones, including OSNet~\cite{zhou2019omni}, MobileNetV2~\cite{sandler2018mobilenetv2}, ResNet50~\cite{resnet}, and HRNet~\cite{dou2022reliability}.
The DCS-Attention modules are inserted after each convolution stage of the CNNs.
We also apply RIB-DCS to a transformer network designed for person Re-ID, TransReID~\cite{he2021transreid}.
Each self-attention module in TransReID is replaced with a DCS-Attention module.
The RIB-DCS-FB models based on the above backbones are referred to as $\textup{RIB-DCS-FB}\,(\cdot)$ where the corresponding backbone name is in the parenthesis.
We compare the performance of our RIB-DCS-FB models with existing state-of-the-art (SOTA) methods for person Re-ID on three datasets. The compared person Re-ID methods are HACNN \cite{li2018harmonious}, OSNet \cite{zhou2019omni}, Auto-ReID \cite{chen2019abd}, RGA \cite{zhang2020relation}, AANet \cite{tay2019aanet}, CAMA \cite{yang2019cama}, BAT-Net \cite{fang2019bilinear}, ABD-Net \cite{chen2019abd}, PAT \cite{li2021diverse}, AutoLoss-GMS \cite{gu2022autoloss}, UAL \cite{dou2022reliability}, BPBreID \cite{somers2023body}, PFD \cite{wang2022pose}, TransReID \cite{he2021transreid}, and PHA \cite{zhang2023pha}.
Table~\ref{SOTA} shows the performance of RIB-DCS-FB models based on OSNet, MobileNetV2, MobileNetv2-200, ResNet, HRNet, and TransReID on all three person Re-ID datasets. We evaluate the performance of all RIB-DCS-FB models with backbones pretrained on ImageNet. For the compact RIB-DCS-FB models, including RIB-DCS-FB (OSNet), RIB-DCS-FB (MobileNetV2), and RIB-DCS-FB (MobileNetv2-200), we also evaluate their performance by training them from scratch on the person Re-ID datasets.
It is observed in Table~\ref{SOTA} that all the RIB-DCS-FB models outperform existing SOTA methods using the same corresponding feature learning backbones. The best RIB-DCS-FB model, RIB-DCS-FB (TransReID), achieves SOTA results on all three person Re-ID datasets. For example, RIB-DCS-FB (TransReID) outperforms PHA, which also uses TransReID as the feature learning backbone, by $1.3\%$ in mAP for person Re-ID on MSMT17.

\noindent\textbf{Results for RIB-DCS-DNAS.}
 Differentiable Neural Architecture Search (DNAS) methods have demonstrated considerable potential in deriving neural networks that are both efficient and effective~\cite{wan2020fbnetv2, autoreid}. Therefore, we apply RIB-DCS-DNAS to search for architectures of efficient DNNs incorporating DCS-Attention that achieve compelling performance for person Re-ID. The evaluation is performed with three versions of the supernets based on FBNetV2. For comparison, we also evaluate the performance of the architectures searched with these three versions of supernets without RIB-DCS.
It is observed in Table~\ref{SOTA-DNAS} that all RIB-DCS-DNAS models significantly outperform corresponding baseline DNAS models. For instance, RIB-DCS-DNAS (FBNetV2-XLarge) outperforms FBNetV2-XLarge by $1.6\%$ in mAP on MSMT17. In addition, with a small number of FLOPs, RIB-DCS-DNAS models outperform existing methods that require significantly more computational resources. Notably, RIB-DCS-DNAS (FBNetV2-XLarge), with only $1.898$G FLOPs, outperforms ABD-Net with $14.1$G FLOPs by $2.1\%$ in mAP on the MSMT17 dataset, demonstrating the efficiency and effectiveness of RIB-DCS-DNAS models.
\begin{table*}[!htbp]
\centering
\caption{Performance of RIB-CA-FB models with the state-of-the-art fixed backbone, TransReID, on Market1501, DukeMTMC, and MSMT17. The improvements of the RIB models over their corresponding baselines are reported in parentheses after each metric value.}
\vspace{-2mm}
\label{tab:rib_channel_fixed}
\resizebox{0.9\linewidth}{!}{
\begin{tabular}{lcccccccc}
\toprule
\multirow{2}{*}{Method} & \multirow{2}{*}{Params(M)} & \multirow{2}{*}{FLOPs(G)} & \multicolumn{2}{c}{Market1501} & \multicolumn{2}{c}{DukeMTMC} & \multicolumn{2}{c}{MSMT17} \\
\cmidrule{4-9}
 &&& mAP & Rank-1 & mAP & Rank-1 & mAP & Rank-1 \\
\midrule
TransReID+SE~\cite{hu2018squeeze}
& 100.6 & 19.5 & 89.5 & 95.6 & 82.6 & 90.9 & 68.0 & 85.5 \\

\textbf{RIB-CA-FB (TransReID+SE)}
& 100.8 & 19.8
& {90.6} ($\uparrow$\,1.1)
& {96.6} ($\uparrow$\,1.0)
& {83.7} ($\uparrow$\,1.1)
& {91.8} ($\uparrow$\,0.9)
& {69.1} ($\uparrow$\,1.1)
& {86.5} ($\uparrow$\,1.0) \\
\midrule

TransReID+CBAM~\cite{CBAM}
& 100.9 & 19.7 & 89.7 & 95.8 & 82.8 & 91.1 & 68.4 & 85.9 \\

\textbf{RIB-CA-FB (TransReID+CBAM)}
& 101.3 & 20.0
& {90.9} ($\uparrow$\,1.2)
& {96.7} ($\uparrow$\,0.9)
& {83.9} ($\uparrow$\,1.1)
& \textbf{92.0} ($\uparrow$\,0.9)
& \textbf{69.5} ($\uparrow$\,1.1)
& \textbf{87.0} ($\uparrow$\,1.1) \\
\midrule

TransReID+MCA~\cite{MCA}
& 101.6 & 19.9 & 90.1 & 96.0 & 83.0 & 91.2 & 68.4 & 85.7 \\

\textbf{RIB-CA-FB (TransReID+MCA)}
& 102.1 & 20.3
& \textbf{91.0} ($\uparrow$\,0.9)
& \textbf{96.8} ($\uparrow$\,0.8)
& \textbf{84.1} ($\uparrow$\,1.1)
& {91.9} ($\uparrow$\,0.7)
& \textbf{69.5} ($\uparrow$\,1.1)
& {86.8} ($\uparrow$\,1.1) \\

\bottomrule
\end{tabular}
}
\end{table*}

\begin{table*}[!htbp]
\centering
\caption{Performance of RIB-CA-DNAS with FBNetV2, FBNetV2-Large, and FBNetV2-XLarge on Market1501, DukeMTMC, and MSMT17. The improvements of the RIB models over their corresponding baselines are reported in parentheses after each metric value.}
\vspace{-2mm}
\label{tab:rib_channel_learnable}
\resizebox{0.95\linewidth}{!}{
\begin{tabular}{lcccccccc}
\toprule
\multirow{2}{*}{Method} & \multirow{2}{*}{Params(M)} & \multirow{2}{*}{FLOPs(G)} &
\multicolumn{2}{c}{Market1501} & \multicolumn{2}{c}{DukeMTMC} &
\multicolumn{2}{c}{MSMT17} \\
\cmidrule{4-9}
 &&& mAP & Rank-1 & mAP & Rank-1 & mAP & Rank-1 \\
\midrule
FBNetV2+MCA~\cite{MCA}
& 5.97 & 0.586 & 84.8 & 94.3 & 76.2 & 87.7 & 58.4 & 80.8 \\

\textbf{RIB-CA-DNAS (FBNetV2+MCA)}
& 6.41 & 0.603
& {85.6} ($\uparrow$\,0.8)
& {95.1} ($\uparrow$\,0.8)
& {76.9} ($\uparrow$\,0.7)
& {88.8} ($\uparrow$\,1.1)
& {59.3} ($\uparrow$\,0.9)
& {81.8} ($\uparrow$\,1.0) \\
\midrule

FBNetV2-Large+MCA~\cite{MCA}
& 13.4 & 1.262 & 85.9 & 95.2 & 78.0 & 89.1 & 60.2 & 82.5 \\

\textbf{RIB-CA-DNAS (FBNetV2-Large+MCA)}
& 13.8 & 1.298
& {86.8} ($\uparrow$\,0.9)
& {95.9} ($\uparrow$\,0.7)
& {79.0} ($\uparrow$\,1.0)
& {89.9} ($\uparrow$\,0.8)
& {60.9} ($\uparrow$\,0.7)
& {82.9} ($\uparrow$\,0.4) \\
\midrule

FBNetV2-XLarge+MCA~\cite{MCA}
& 24.9 & 1.868 & 88.5 & 95.5 & 79.8 & 90.2 & 61.7 & 82.9 \\

\textbf{RIB-CA-DNAS (FBNetV2-XLarge+MCA)}
& 25.5 & 1.985
& \textbf{89.7} ($\uparrow$\,1.2)
& \textbf{96.2} ($\uparrow$\,0.7)
& \textbf{80.6} ($\uparrow$\,0.8)
& \textbf{91.1} ($\uparrow$\,0.9)
& \textbf{62.8} ($\uparrow$\,1.1)
& \textbf{83.5} ($\uparrow$\,0.6) \\

\bottomrule
\end{tabular}}
\vspace{-2mm}
\end{table*}

\vspace{-2mm}

\subsection{RIB-CA for Person Re-ID}
\label{sec:ChannelAttn_Results}

\noindent\textbf{Results for RIB-CA-FB.}
We apply RIB-CA to the SOTA person Re-ID backbone, TransReID \cite{he2021transreid}. We first insert existing channel attention modules, including Squeeze-and-Excitation (SE)~\cite{hu2018squeeze}, Convolutional Block Attention Module (CBAM)~\cite{CBAM}, and Moment Channel Attention (MCA)~\cite{MCA}, after each transformer block in TransReID, leading to TransReID+SE, TransReID+CBAM, and TransReID+MCA.
The RIB-CA models based on TransReID+SE, TransReID+CBAM, and TransReID+MCA are denoted as RIB-CA-FB (TransReID+SE), RIB-CA-FB (TransReID+CBAM), and RIB-CA-FB (TransReID+MCA).
It is observed in Table~\ref{tab:rib_channel_fixed} that the RIB-CA-FB models consistently outperform their corresponding baseline models across all the person Re-ID benchmarks. For instance, RIB-CA-FB (TransReID+CBAM) outperforms TransReID+CBAM by $1.2\%$ in mAP for person Re-ID on Market1501, demonstrating the effectiveness of learning informative channel attention weights with RIB-CA.

\noindent\textbf{Results for RIB-CA-DNAS.}
We further evaluate RIB-CA under the setting where existing channel attention modules are integrated into the FBNetV2 supernet, after which the backbone architecture is optimized using the DNAS algorithm. The experiments are performed using the best channel attention module, MCA~\cite{MCA}.
The models incorporating MCA based on FBNetV2, FBNetV2-Large, and FBNetV2-XLarge are denoted as FBNetV2+MCA, FBNetV2-Large+MCA, and FBNetV2-XLarge+MCA. It is observed in Table~\ref{tab:rib_channel_learnable} that the RIB-CA-DNAS models consistently outperform the baseline models incorporating existing attention modules into the DNAS backbones. For instance, RIB-CA-DNAS (FBNetV2-XLarge+MCA) outperforms FBNetV2-XLarge+MCA by $1.2\%$ in mAP on Market1501, demonstrating the advantages of learning informative channel attention weights with learnable backbones.

\begin{table}[!htbp]
 \centering
  \caption{Ablation study on the effects of IBB. The comparisons are performed on Market1501.}
  \vspace{-2mm}
 \label{tab:ablation-IB-loss}
 \resizebox{\columnwidth}{!}{
\begin{tabular}{lcccc}
\toprule
Model                    & mAP  &Rank-1       & IBB           & IB Loss  \\
\midrule
MobileNetV2     & 84.5 & 93.9  & 0.062  & -0.018     \\
MobileNetV2+Self-Attention   & 85.0 & 94.1  & 0.060  & -0.019    \\
MobileNetV2+DCS-Attention   & 85.8 & 94.5 & 0.055   & -0.027  \\
\textbf{RIB-DCS-FB (MobileNetV2)}  & \textbf{86.4} & \textbf{95.1} & \textbf{0.028}& \textbf{-0.064} \\
\midrule
TransReID    & 88.9 & 95.2  & 0.059  & -0.020     \\
TransReID+DCS-Attention   & 90.4 & 95.8 & 0.050   & -0.028  \\
\textbf{RIB-DCS-FB (TransReID)}   & \textbf{91.3} & \textbf{97.0} & \textbf{0.027}& \textbf{-0.067} \\
\midrule
FBNetV2-Large    & 87.2 &95.0  & 0.061  & -0.019     \\
FBNetV2-Large+Self-Attention  & 87.9 &95.2  & 0.059  & -0.020    \\
FBNetV2-Large+DCS-Attention  & 88.7 & 95.4 & 0.053   & -0.028  \\
\textbf{RIB-DCS-DNAS (FBNetV2-Large)}   & \textbf{89.1} & \textbf{95.8} & \textbf{0.028}& \textbf{-0.065} \\
\bottomrule
\end{tabular}
 }
\vspace{-3mm}
\end{table}

\subsection{Ablation Study on the Effects of IBB}
\label{sec:ablation_IB_loss}
In this section, we study the effectiveness of the variational upper bound for the IB loss, which is IBB, in the training of RIB-DCS-FB and RIB-DCS-DNAS models for person Re-ID.
 The comparisons are performed on Market1501 with three base models, which are MobileNetV2, TransReID, and FBNetV2-Large.
 Since the vanilla MobileNetV2 and FBNetV2-Large do not contain attention modules, we add two baseline models that incorporate vanilla self-attention modules into MobileNetV2 and FBNetV2-Large, which are referred to as MobileNetV2+Self-Attention and FBNetV2-Large+Self-Attention.
We also evaluate three ablation models, which incorporate the DCS-Attention modules into the backbones of MobileNetV2, TransReID, and FBNetV2-Large. The above ablation models that only incorporate DCS-Attention without reducing IBB are referred to as   MobileNetV2+DCS-Attention, TransReID+DCS-Attention, and FBNetV2-Large+DCS-Attention, respectively.
It is observed in Table~\ref{tab:ablation-IB-loss} that although optimizing the regular cross-entropy loss and the triplet loss to train DNNs with DCS-Attention can decrease the IB loss so that the IB principle, learning features more strongly correlated with class labels while decreasing their correlation with the input, is adhered to, the IB loss can be decreased further by a considerable extent by optimizing the IBB explicitly.
In particular, our RIB-DCS-FB and RIB-DCS-DNAS models outperform the baseline models and the ablation models by large margins in terms of IB loss, mAP, and Rank-1 accuracy. For instance, RIB-DCS-FB (TransReID) further decreases the IB loss of TransReID+DCS-Attention by $0.039$, leading to a $0.9\%$ increase in mAP and a $1.2\%$ increase in Rank-1 accuracy for person Re-ID on Market1501.

\subsection{Comparison between IBB and Existing Upper Bounds for the IB Loss}
\label{sec:ablation_IBB}
We compare the proposed variational upper bound for the IB loss, IBB, with existing works deriving the upper bound for the IB loss~\cite{CLUB, VIB-DaiZGW18, VIB-SrivastavaDGAA21}.
Although CLUB~\cite{CLUB} also proposes an upper bound for the IB loss, the derivation of the upper bound in CLUB assumes that the $p(F|X)$ is known, where $F$ and $X$ are the random variables representing the learned feature and the input feature. When $p(F|X)$ is unknown, CLUB adopts a Gaussian Mixture Model (GMM) parameterized by an external neural network to approximate the upper bound. In contrast, our novel variational upper bound, $\textup{IBB}$, for the IB loss does not have such an assumption. As shown in Lemma~\ref{lemma:I-X-tildeX-upper-bound} in Section~\ref{sec:proofs} of the supplementary, $p(F|X)$ can be directly computed from the training data without the need for training another neural network. Let $T_0$ be the computational complexity of a forward and backward computation for the neural network with regard to each training data.
The computational complexity for calculating the IBB is only $\Theta(nCT_0 + C^2)$.
In contrast, computing the upper bound for the IB loss proposed in CLUB~\cite{CLUB} requires a substantially higher computational complexity of more than $\Theta(n^2 T_0)$. We note that $\Theta(n^2 T_0)$ corresponds exclusively to the upper bound for the mutual information $I(F,X)$, while CLUB~\cite{CLUB} additionally requires the computation of the lower bound of the mutual information $I(F,Y)$ so as to find the upper bound for the IB loss $I(F,X)-I(F,Y)$.
Details on the complexity analysis of CLUB and IBB are deferred to Section~\ref{sec:complexity} of the supplementary.

In addition, the upper bound for the IB loss proposed in VIB~\cite{VIB-DaiZGW18, VIB-SrivastavaDGAA21} imposes an unrealistic Gaussian distribution assumption on the hidden features of the DNNs. Instead of reducing the IB loss or its upper bound, APIB~\cite{IB-lasso-APIB} reduces only an approximation to the IB loss.
To compare the performance of IBB with VIB~\cite{VIB-DaiZGW18, VIB-SrivastavaDGAA21} and APIB~\cite{IB-lasso-APIB}, we conduct an ablation study by replacing IBB with VIB, APIB, and CLUB for RIB-DCS-FB (TransReID) for person Re-ID on Market1501, with the total training time on a single NVIDIA A100 GPU reported for our method and the competing baselines.
It is observed in Table~\ref{tab:compare-IBB-VIB-APIB} that the unrealistic Gaussian distribution assumption on the hidden features imposed by VIB~\cite{VIB-DaiZGW18, VIB-SrivastavaDGAA21} and the IB approximation in APIB~\cite{IB-lasso-APIB} lead to degraded performance compared with our reduction of IBB. For instance, the model based on IBB outperforms the model based on VIB by $1.4\%$ in mAP on Market1051. Moreover, the model based on IBB outperforms the model based on CLUB by $1.0\%$ in mAP while taking only $56.1\%$ of the training time of the model based on CLUB.

\begin{table}[!htbp]
\centering
\caption{Comparison of different methods for reducing the IB loss. The study is performed on Market1501 with the feature backbone RIB-DCS-FB (TransReID).}
\label{tab:compare-IBB-VIB-APIB}
\vspace{-2mm}
\resizebox{0.95\columnwidth}{!}{
\begin{tabular}{lccc}
\toprule
{Methods}  & Training Time (Hours)  & {mAP} & {Rank-1}  \\
\midrule
VIB~\cite{VIB-DaiZGW18, VIB-SrivastavaDGAA21} & 32.2  & 89.9   & 95.8 \\
APIB~\cite{IB-lasso-APIB} & 32.6 & 90.1 &  96.0 \\
CLUB~\cite{CLUB} & 58.6  & 90.3 & 95.9 \\
\textbf{IBB (Ours)} & 32.9 & \textbf{91.3}  & \textbf{97.0} \\
\bottomrule
\end{tabular}
}
\end{table}

\vspace{-2mm}
\subsection{Evaluation on Occluded-Duke Dataset}
\label{sec:occluded_duke}
To further verify the robustness of the proposed method for person Re-ID under occlusion, we conduct experiments on the Occluded-Duke~\cite{miao2019pose} dataset. Images in Occluded-Duke contain severe partial occlusions and spatial misalignments caused by diverse viewing angles and pedestrian overlaps, making person re-identification particularly challenging. We evaluate both TransReID~\cite{he2021transreid} and SPT~\cite{tan2024occluded}, along with the RIB-DCS-FB models using TransReID and SPT as backbones, under the same training and inference configuration described in Section~\ref{sec:implementation_details}.
SPT is the state-of-the-art occluded person Re-ID model based on a transformer architecture that uses vanilla self-attention modules similar to TransReID.
It is observed in Table~\ref{tab:occluded_duke} that RIB-DCS significantly improves the performance of both TransReID and SPT, consistently achieving higher mAP and Rank-1 accuracy.
For example, RIB-DCS-FB (SPT) outperforms SPT by $1.2\%$ in mAP, which demonstrates that RIB effectively learns informative and semantically relevant representations that are robust to the occlusion and background in the input image.

\begin{table}[!htbp]
\centering
\caption{Performance comparison on Occluded-Duke under the same training and inference settings without re-ranking.}
\vspace{-2mm}

\resizebox{0.725\columnwidth}{!}{
\begin{tabular}{lcc}
\toprule
{Method}  & {mAP} & {Rank-1} \\
\midrule
TransReID~\cite{he2021transreid}  & 59.2 & 66.4 \\
\textbf{RIB-DCS-FB (TransReID)}   & \textbf{60.5} & \textbf{67.9} \\
\midrule
SPT~\cite{tan2024occluded}  & 63.0 & 74.7 \\
\textbf{RIB-DCS-FB (SPT)}  & \textbf{64.2} & \textbf{75.7} \\
\bottomrule
\end{tabular}
}
\label{tab:occluded_duke}
\vspace{-4mm}
\end{table}
\subsection{Training Time Evaluation}
\label{sec:training_time}
In this section, we assess the training costs of our RIB-DCS-FB and RIB-DCS-DNAS models and the corresponding baseline models on Market1501. This evaluation is performed on one NVIDIA A100 GPU, utilizing an effective batch size of 256 images. The total training duration for 600 epochs is reported. It is observed in Table~\ref{tab:training_time} that our RIB-DCS-FB and RIB-DCS-DNAS models only marginally increase the training time of the corresponding baselines while significantly improving the performance. For example, the training time of RIB-DCS-FB (TransReID) is only $5.7\%$ longer than the training time of TransReID, while RIB-DCS-FB (TransReID) outperforms TransReID by $2.4\%$ in mAP on Market1501.
\begin{table}[!htbp]
 \centering
  \caption{Training time (hours) comparisons on Market1501 between RIB-DCS-FB (MobileNetV2), RIB-DCS-FB (TransReID), RIB-DCS-DNAS (FBNetV2-Large), and their corresponding baseline models.}
  \vspace{-2mm}
 \label{tab:training_time}
 \resizebox{1\columnwidth}{!}{
\begin{tabular}{lcccc}
\toprule
Model                   & Training Time  & mAP  &R1      \\
\midrule
MobileNetV2    & 10.3 & 84.5& 93.9     \\
\textbf{RIB-DCS-FB (MobileNetV2)}   & 10.9 & \textbf{86.4} & \textbf{95.1}  \\
\midrule
TransReID    & 31.1  & 88.9 & 95.2      \\
\textbf{RIB-DCS-FB (TransReID)}  & 32.9 & \textbf{91.3} & \textbf{97.0}  \\
\midrule
FBNetV2-Large & 16.7    & 87.2 &95.0 \\
\textbf{RIB-DCS-DNAS (FBNetV2-Large)}  & 17.5    & \textbf{89.1} &\textbf{95.8}  \\
\bottomrule
\end{tabular}
 }
\vspace{-4mm}
\end{table}

\subsection{Search Time Evaluation}
\label{sec:search_time}
In this section, we assess the search costs of our RIB-DCS-DNAS and RIB-CA-DNAS models and the corresponding baseline models, which do not incorporate IBB into the training loss. The evaluation is conducted on Market1501 using one NVIDIA A100 GPU. It is observed in Table~\ref{tab:search_time} that our RIB only marginally increases the search time of the baseline models, FBNetV2-XLarge+MCA and FBNetV2-XLarge, while significantly increasing their performance for person Re-ID. For example, RIB-DCS-DNAS (FBNetV2-XLarge) requires only an additional $7.2\%$ of search time compared to FBNetV2-XLarge, while achieving a significant improvement of $1.2\%$ in mAP.

\begin{table}[!htbp]
 \centering
  \caption{Searching time (hours) comparisons on Market1501 between RIB-CA-DNAS (FBNetV2-XLarge+MCA), RIB-DCS-DNAS (FBNetV2-XLarge) and their corresponding baseline models.}
  \vspace{-2mm}
 \label{tab:search_time}
 \resizebox{1\columnwidth}{!}{
\begin{tabular}{lcccc}
\toprule
Model                   & Training Time  & mAP  &R1      \\
\midrule
FBNetV2-XLarge+MCA & 2.26    & 87.9 &95.3 \\
\textbf{RIB-CA-DNAS (FBNetV2-XLarge+MCA)}  & 2.36    & \textbf{89.8} &\textbf{96.4}  \\
\midrule
FBNetV2-XLarge & 2.22    & 88.5 &95.5 \\
\textbf{RIB-DCS-DNAS (FBNetV2-XLarge)}  & 2.38    & \textbf{89.7} &\textbf{96.2}  \\
\bottomrule
\end{tabular}
 }
 \vspace{-4mm}
\end{table}
\section{Conclusion}
In this paper, we propose a novel distribution-free and efficient Information Bottleneck (IB) reduction framework, termed RIB. Motivated by the IB principle, RIB aims to learn informative attention weights for DNNs with self-attention and channel attention modules by explicitly reducing the IB loss.
To this end, we present and prove a novel distribution-free and efficient variational upper bound for the IB loss, termed IBB, which can be optimized by standard SGD algorithms. To apply RIB to self-attention modules, we develop a novel Differentiable Channel Selection Attention method, or DCS-Attention, which selects more informative feature channels to compute more informative attention weights than vanilla self-attention, leading to RIB-DCS.
We also apply RIB to DNNs to existing channel attention modules, leading to RIB-CA which encourages learning informative channel attention weights. Both RIB-DCS and RIB-CA are applied to fixed and learnable backbones. Extensive experiments, including occluded person Re-ID, cross-domain person Re-ID, and self-supervised person Re-ID, show that the proposed RIB framework consistently improves the performance of DNNs with self-attention and channel attention.


%

\appendices

\ifCLASSOPTIONcaptionsoff
  \newpage
\fi

\begin{appendices}
\renewcommand{\thesection}{\Alph{section}}
\renewcommand{\thesubsection}{\thesection.\arabic{subsection}}
\renewcommand{\thesubsubsection}{\thesubsection.\arabic{subsubsection}}

\section{Additional Experiment Results}
\label{sec:add_exp_sup}
\subsection{Details on the Datasets}
\label{sec:dataset_sup}
Market-1501 \cite{market} consists of $32,668$ annotated person images of $1,501$ identities, in which $12,936$ images of $751$ identities are used for training and $19,732$ images of $750$ identities are used for testing. All images are shot with $6$ different cameras. DukeMTMC \cite{duke} was originally proposed for video-based Re-ID. It has $16,522$ training images of $702$ identities. The other $2,228$ query images of $702$ identities and $17,661$ gallery images are in the test set. MSMT17 \cite{msmt} is the largest and the most challenging public person re-ID dataset, which includes $126,441$ person images of $4,101$ identities detected by Faster R-CNN \cite{ren2015faster}. The training set has $32,621$ person images of $1,041$ identities, and the test set has $93,820$ images of the other $3,060$ identities. CUHK03 \cite{cuhk03} consists of $13,164$ person images of $1,467$ identities, of which images of $767$ identities are in the training set, and the remaining $700$ identities are in the test set.
\subsection{Additional Implementation Details}
\label{sec:FBNet_ARCH}
\noindent\textbf{Tuning Hyper-Parameters by Cross-Validation.} To decide the best balancing factor $\eta$ for the IBB in the overall training loss, we perform 5-fold cross-validation on $10\%$ of the training data. The value of $\eta$ is selected from $\{0.1, 0.25, 0.5, 0.75, 1, 1.25, 1.5, 1.75, 2\}$. The validation results suggest that $\eta=1$ works best for all our RIB-DCS and RIB-CA models on different datasets.

\noindent\textbf{FBNetV2 Supernet Structures.} FBNetV2~\cite{wan2020fbnetv2} employs the Differentiable Neural Architecture Search (DNAS) algorithm to learn the backbone architecture by selecting neural options in a supernet. The search space of FBNetV2 features a masking mechanism based on Gumbel-Softmax for feature map reuse so that it can efficiently search for the number of filters of each convolution layer. In addition, we designed two larger supernet architectures of FBNetV2, termed FBNetV2-Large and FBNetV2-XLarge. The depth of FBNetV2-Large is $3$ times the depth of the original supernet of FBNetV2. The depth of FBNetV2-XLarge is $6$ times the depth of the original supernet of FBNetV2.
The structure of the supernet of FBNetV2, FBNetV2-Large, and FBNetV2-XLarge are detailed in Table~\ref{fbnet3} and Table~\ref{fbnet6}, respectively.
\begin{table}[!ht]
\centering
\caption{The supergraph of FBNetV2 and FBNetV2-Large (3  $\times$ depth), with block expansion rate $e$, number of filters $f$, number of blocks $n$, and stride of first block $s$ \,\, Tuples of three values in the column of expansion rate $e$ and number of filters $f$ represent the lowest value, highest, and steps between options (low, high, steps). }
\vspace{-2mm}
\label{fbnet3}
\resizebox{\columnwidth}{!}{
\begin{tabular}{cccccc}
\hline
\multicolumn{1}{c|}{\multirow{2}{*}{Operator}} & \multicolumn{1}{c|}{\multirow{2}{*}{$e$}} & \multicolumn{1}{c|}{\multirow{2}{*}{$f$}} & \multicolumn{2}{c|}{$n$}                                          & \multirow{2}{*}{$s$} \\ \cline{4-5}
\multicolumn{1}{c|}{}                          & \multicolumn{1}{c|}{}                     & \multicolumn{1}{c|}{}                     & \multicolumn{1}{c|}{FBNetV2} & \multicolumn{1}{c|}{FBNetV2-Large} &                      \\ \hline
\multicolumn{1}{c|}{conv2d}                    & \multicolumn{1}{c|}{1}                    & \multicolumn{1}{c|}{16}                   & \multicolumn{1}{c|}{1}       & \multicolumn{1}{c|}{1}             & 2                    \\
\multicolumn{1}{c|}{bottleneck}                & \multicolumn{1}{c|}{1}                    & \multicolumn{1}{c|}{(12, 16, 4)}          & \multicolumn{1}{c|}{1}       & \multicolumn{1}{c|}{1}             & 1                    \\
\multicolumn{1}{c|}{bottleneck}                & \multicolumn{1}{c|}{(0.75, 3.25, 0.5)}    & \multicolumn{1}{c|}{(16, 28, 4)}          & \multicolumn{1}{c|}{1}       & \multicolumn{1}{c|}{1}             & 2                    \\
\multicolumn{1}{c|}{bottleneck}                & \multicolumn{1}{c|}{(0.75, 3.25, 0.5)}    & \multicolumn{1}{c|}{(16, 28, 4)}          & \multicolumn{1}{c|}{2}       & \multicolumn{1}{c|}{6}             & 1                    \\ \hline
\multicolumn{6}{c}{DCS-Attention Module or Channel Attention Module Inserted} \\ \hline
\multicolumn{1}{c|}{bottleneck}                & \multicolumn{1}{c|}{(0.75, 3.25, 0.5)}    & \multicolumn{1}{c|}{(16, 40, 8)}          & \multicolumn{1}{c|}{1}       & \multicolumn{1}{c|}{3}             & 2                    \\
\multicolumn{1}{c|}{bottleneck}                & \multicolumn{1}{c|}{(0.75, 3.25, 0.5)}    & \multicolumn{1}{c|}{(16, 40, 8)}          & \multicolumn{1}{c|}{2}       & \multicolumn{1}{c|}{6}             & 1                    \\ \hline
\multicolumn{6}{c}{DCS-Attention Module or Channel Attention Module Inserted}                                                                                                                                                                                            \\ \hline
\multicolumn{1}{c|}{bottleneck}                & \multicolumn{1}{c|}{(0.75, 3.25, 0.5)}    & \multicolumn{1}{c|}{(48, 96, 8)}          & \multicolumn{1}{c|}{1}       & \multicolumn{1}{c|}{3}             & 2                    \\
\multicolumn{1}{c|}{bottleneck}                & \multicolumn{1}{c|}{(0.75, 3.25, 0.5)}    & \multicolumn{1}{c|}{(48, 96, 8)}          & \multicolumn{1}{c|}{2}       & \multicolumn{1}{c|}{6}             & 1                    \\
\multicolumn{1}{c|}{bottleneck}                & \multicolumn{1}{c|}{(0.75, 4.5, 0.75)}    & \multicolumn{1}{c|}{(72, 128, 8)}         & \multicolumn{1}{c|}{4}       & \multicolumn{1}{c|}{12}            & 1                    \\ \hline
\multicolumn{6}{c}{DCS-Attention Module or Channel Attention Module Inserted}                                                                                                                                                                             \\ \hline
\multicolumn{1}{c|}{bottleneck}                & \multicolumn{1}{c|}{(0.75, 4.5, 0.75)}    & \multicolumn{1}{c|}{(112, 216, 8)}        & \multicolumn{1}{c|}{1}       & \multicolumn{1}{c|}{3}             & 2                    \\
\multicolumn{1}{c|}{bottleneck}                & \multicolumn{1}{c|}{(0.75, 4.5, 0.75)}    & \multicolumn{1}{c|}{(112, 216, 8)}        & \multicolumn{1}{c|}{3}       & \multicolumn{1}{c|}{3}             & 1                    \\ \hline
\multicolumn{6}{c}{DCS-Attention Module or Channel Attention Module Inserted}                                                                                                                                                                         \\ \hline
\multicolumn{1}{c|}{conv2d}                    & \multicolumn{1}{c|}{-}                    & \multicolumn{1}{c|}{1984}                 & \multicolumn{1}{c|}{1}       & \multicolumn{1}{c|}{1}             & 1                    \\
\multicolumn{1}{c|}{fc}                        & \multicolumn{1}{c|}{-}                    & \multicolumn{1}{c|}{-}                    & \multicolumn{1}{c|}{1}       & \multicolumn{1}{c|}{-}             & -                    \\ \hline
\end{tabular}
}
\vspace{-2mm}
\end{table}

\begin{table}[!htbp]
\centering
\caption{The supergraph of FBNetV2-XLarge (6  $\times$ depth), with block expansion rate $e$, number of filters $f$, number of blocks $n$, and stride of first block $s$ \,\, Tuples of three values in the column of expansion rate $e$ and number of filters $f$ represent the lowest value, highest, and steps between options (low, high, steps). }
\label{fbnet6}
\resizebox{0.85\columnwidth}{!}{
\begin{tabular}{ccccc}
\hline
\multicolumn{1}{c|}{\multirow{2}{*}{Operator}} & \multicolumn{1}{c|}{\multirow{2}{*}{$e$}} & \multicolumn{1}{c|}{\multirow{2}{*}{$f$}} & \multicolumn{1}{c|}{\multirow{2}{*}{$n$}} & \multirow{2}{*}{$s$} \\
\multicolumn{1}{c|}{}                          & \multicolumn{1}{c|}{}                     & \multicolumn{1}{c|}{}                     & \multicolumn{1}{c|}{}                     &                      \\ \hline
\multicolumn{1}{c|}{conv2d}                    & \multicolumn{1}{c|}{1}                    & \multicolumn{1}{c|}{16}                   & \multicolumn{1}{c|}{1}                    & 2                    \\
\multicolumn{1}{c|}{bottleneck}                & \multicolumn{1}{c|}{1}                    & \multicolumn{1}{c|}{(12, 16, 4)}          & \multicolumn{1}{c|}{1}                    & 1                    \\
\multicolumn{1}{c|}{bottleneck}                & \multicolumn{1}{c|}{(0.75, 6.25, 0.5)}    & \multicolumn{1}{c|}{(16, 28, 4)}          & \multicolumn{1}{c|}{1}                    & 2                    \\
\multicolumn{1}{c|}{bottleneck}                & \multicolumn{1}{c|}{(0.75, 6.25, 0.5)}    & \multicolumn{1}{c|}{(16, 28, 4)}          & \multicolumn{1}{c|}{12}                   & 1                    \\ \hline
\multicolumn{5}{c}{DCS-Attention Module or Channel Attention Module Inserted}                                                                                                                                                                    \\ \hline
\multicolumn{1}{c|}{bottleneck}                & \multicolumn{1}{c|}{(0.75, 6.25, 0.5)}    & \multicolumn{1}{c|}{(16, 40, 8)}          & \multicolumn{1}{c|}{6}                    & 2                    \\
\multicolumn{1}{c|}{bottleneck}                & \multicolumn{1}{c|}{(0.75, 6.25, 0.5)}    & \multicolumn{1}{c|}{(16, 40, 8)}          & \multicolumn{1}{c|}{12}                   & 1                    \\ \hline
\multicolumn{5}{c}{DCS-Attention Module or Channel Attention Module Inserted}                                                                                                                                                                    \\ \hline
\multicolumn{1}{c|}{bottleneck}                & \multicolumn{1}{c|}{(0.75, 6.25, 0.5)}    & \multicolumn{1}{c|}{(48, 96, 8)}          & \multicolumn{1}{c|}{6}                    & 2                    \\
\multicolumn{1}{c|}{bottleneck}                & \multicolumn{1}{c|}{(0.75, 6.25, 0.5)}    & \multicolumn{1}{c|}{(48, 96, 8)}          & \multicolumn{1}{c|}{12}                   & 1                    \\
\multicolumn{1}{c|}{bottleneck}                & \multicolumn{1}{c|}{(0.75, 7.5, 0.75)}    & \multicolumn{1}{c|}{(72, 128, 8)}         & \multicolumn{1}{c|}{24}                   & 1                    \\ \hline
\multicolumn{5}{c}{DCS-Attention Module or Channel Attention Module Inserted}                                                                                                                                                                    \\ \hline
\multicolumn{1}{c|}{bottleneck}                & \multicolumn{1}{c|}{(0.75, 7.5, 0.75)}    & \multicolumn{1}{c|}{(112, 216, 8)}        & \multicolumn{1}{c|}{6}                    & 2                    \\
\multicolumn{1}{c|}{bottleneck}                & \multicolumn{1}{c|}{(0.75, 7.5, 0.75)}    & \multicolumn{1}{c|}{(112, 216, 8)}        & \multicolumn{1}{c|}{6}                    & 1                    \\ \hline
\multicolumn{5}{c}{DCS-Attention Module or Channel Attention Module Inserted}                                                                                                                                                                    \\ \hline
\multicolumn{1}{c|}{conv2d}                    & \multicolumn{1}{c|}{-}                    & \multicolumn{1}{c|}{1984}                 & \multicolumn{1}{c|}{1}                    & 1                    \\
\multicolumn{1}{c|}{avgpool}                   & \multicolumn{1}{c|}{-}                    & \multicolumn{1}{c|}{-}                    & \multicolumn{1}{c|}{1}                    & 1                    \\
\multicolumn{1}{c|}{fc}                        & \multicolumn{1}{c|}{-}                    & \multicolumn{1}{c|}{-}                    & \multicolumn{1}{c|}{-}                    & -                    \\ \hline
\end{tabular}
}
\vspace{-2mm}
\end{table}

\subsection{Comparison with Existing Attention Masking Method}
\label{sec:moda}
In this section, we compare RIB-DCS with an existing attention masking method, MODA~\cite{MODA}, which is proposed to mitigate attention imbalance in multimodal Transformers by introducing a non-learnable modular attention mask. MODA constructs its mask by combining a predefined decay prior with pseudo-attention statistics extracted from intermediate activations, assigning progressively larger negative penalties to token pairs based on their positional distance in the token sequence. In the single-modality setting, such as person Re-ID, MODA can be applied by treating all visual patches as a single token stream. The same decay mechanism is used to regulate how strongly a patch is allowed to attend to patches that are sequentially distant. As a result, MODA produces a mask on the attention weights that enforces a token-level structural constraint that suppresses excessive long-range attention and promotes more balanced attention distributions.

\begin{table}[!htbp]
 \centering
  \caption{Comparison between RIB-DCS and existing attention masking method, MODA~\cite{MODA}, for person Re-ID on Market1501.}
 \label{tab:MODA_compare}
\resizebox{0.8\columnwidth}{!}{
\begin{tabular}{lccccc}
\toprule
Model                   & FLOPs  & mAP  &Rank-1      \\
\midrule
TransReID    & 19.3 G  & 88.9 & 95.2      \\
MODA (TransReID)   & 19.4 G & 89.4 & 95.3 \\
\textbf{RIB-DCS-FB (TransReID)}  &19.7 G & \textbf{91.3} & \textbf{97.0}  \\
\bottomrule
\end{tabular}
}
\end{table}

However, such a positional decay prior may not be beneficial for person Re-ID, where discriminative visual patterns are often scattered across spatially distant body regions. For example, the head, clothes, backpack, and shoes of a person may occupy distant patches in the token sequence, while jointly providing essential identity-critical information. Penalizing interactions among these patches that are far apart risks weakening meaningful long-range dependencies and degrading the model's ability to capture informative features related to the person's identity.
In contrast with MODA, our proposed RIB-DCS aims to learn an informative weight matrix that makes the learned features more correlated with the class labels and less correlated with the input features, without relying on unreliable priors.
Instead of assigning penalties based on spatial proximity, RIB-DCS learns to remove uninformative channels while preserving those that carry identity-relevant information. The mask module is optimized jointly with the network parameters under the guidance of RIB. To demonstrate the advantages of RIB-DCS over MODA, we perform an experiment by applying MODA to all the attention modules in TransReID, leading to MODA (TransReID). It is observed in Table~\ref{tab:MODA_compare} that RIB-DCS-FB (TransReID) significantly outperforms MODA (TransReID) by $1.9\%$ in mAP and $1.7\%$ in Rank-1 for person Re-ID on Market1501.

\subsection{Evaluation on CUHK03 Dataset}
\label{sec:cuhk03}
We further evaluate the effectiveness of the proposed RIB on the CUHK03~\cite{cuhk03} dataset. The labeled version of the dataset, where the bounding boxes of the persons are manually annotated by humans, is used following the configurations in~\cite{cuhk03}, without re-ranking or multi-query fusion. We compare the performance of RIB-DCS-FB (TransReID), RIB-DCS-DNAS (FBNetV2-XLarge), RIB-CA-FB (TransReID+MCA), and RIB-CA-DNAS (FBNetV2-XLarge+MCA) with their corresponding baseline models without RIB. It is observed in Table~\ref{tab:cuhk03} that RIB significantly improves the performance of all the baseline models. For instance, RIB-DCS-FB (TransReID) outperforms TransReID by $1.6\%$ in mAP for

\begin{table}[!htbp]
\centering
\caption{Performance comparison on the CUHK03 dataset under the traditional evaluation protocol. Performance improvements over baseline models are indicated in parentheses.}
\resizebox{0.95\columnwidth}{!}{
\begin{tabular}{lcc}
\toprule
{Method}  & {mAP (\%)} & {Rank-1 (\%)} \\
\midrule
TransReID~\cite{he2021transreid}  & 79.6 & 81.7 \\
\textbf{RIB-DCS-FB (TransReID)} & \textbf{81.2} ($\uparrow$\,1.6) & \textbf{82.7} ($\uparrow$\,1.0) \\
\midrule
FBNetV2-XLarge~\cite{wan2020fbnetv2}  & 78.0 & 80.3 \\
\textbf{RIB-DCS-DNAS (FBNetV2-XLarge)} & \textbf{79.6} ($\uparrow$\,1.6) & \textbf{81.1} ($\uparrow$\,0.8) \\
\midrule
TransReID+MCA~\cite{he2021transreid}  & 80.1 & 81.8 \\
\textbf{RIB-CA-FB (TransReID+MCA)} & \textbf{81.1} ($\uparrow$\,1.0) & \textbf{82.5} ($\uparrow$\,0.7) \\
\midrule
FBNetV2-XLarge+MCA~\cite{wan2020fbnetv2} & 78.4 & 80.6 \\
\textbf{RIB-CA-DNAS (FBNetV2-XLarge+MCA)} & \textbf{79.7} ($\uparrow$\,1.3) & \textbf{81.3} ($\uparrow$\,0.7) \\
\bottomrule
\end{tabular}
}
\label{tab:cuhk03}
\end{table}

\subsection{Cross-Domain Person Re-ID}
\label{sec:cross_dataset}
To further evaluate the cross-domain generalization capability of the proposed RIB method, we integrate RIB-DCS into the domain transfer framework of CaCL~\cite{LeeL0SYH23}, which performs contrastive adaptation across heterogeneous person Re-ID datasets. In particular, CaCL aims to bridge the domain gap between the source and target datasets through class-aware contrastive learning and domain-level feature alignment.
Both our model, RIB-DCS-FB (CaCL), and the baseline model, CaCL, are evaluated under the standard cross-dataset Re-ID protocol~\cite{LeeL0SYH23}, where training is performed on one dataset (source domain), and testing is conducted on another (target domain).
It is observed in Table~\ref{tab:cross_dataset} that RIB-DCS-FB (CaCL) significantly outperforms the corresponding baseline model, CaCL. For instance, RIB-DCS-FB (CaCL) outperforms CaCL by $1.7\%$ in mAP when the models are trained on Market1501 and tested on MSMT17, which demonstrates the capabilities of the RIB in learning informative features from the input image for identifying the person's identity even when the test data is not from the same domain as the training data.

\begin{table}[!htbp]
\centering
\caption{Performance comparison for cross-domain evaluation on MSMT17 and Market1501. Performance improvements over the baseline model CaCL~\cite{LeeL0SYH23} are indicated in parentheses.}
\vspace{-2mm}
\resizebox{1\columnwidth}{!}{
\begin{tabular}{lcccc}
\toprule
 {Method} &{Training} & {Test} & {mAP} & {Rank-1} \\
\midrule
CaCL~\cite{LeeL0SYH23} &MSMT17 & Market1501 &  84.7 & 93.8 \\
\textbf{RIB-DCS-FB (CaCL)} &MSMT17 & Market1501 &  \textbf{85.7} ($\uparrow$\,1.0) & \textbf{94.6} ($\uparrow$\,0.8) \\
\midrule
CaCL~\cite{LeeL0SYH23} & Market1501 & MSMT17 & 66.6 & 36.5 \\
 \textbf{RIB-DCS-FB (CaCL)} &Market1501 & MSMT17 & \textbf{68.3} ($\uparrow$\,1.7) & \textbf{37.7} ($\uparrow$\,1.2) \\
\bottomrule
\end{tabular}
}
\label{tab:cross_dataset}
\end{table}

\subsection{Integration with Self-Supervised Re-IDentification Frameworks}
\label{sec:SSL}
To further demonstrate the general applicability of the proposed RIB, we integrate RIB-DCS into three representative self-supervised person re-identification frameworks, including the CLIP-based method CLIP-REID~\cite{CLIP-REID}, the DINO-based method SOLIDER~\cite{SOLIDER}, and the MAE-based method PersonMAE~\cite{PersonMAE}. These frameworks represent three distinct paradigms of self-supervised representation learning, including vision-language contrastive alignment, self-distillation via teacher-student consistency, and masked image reconstruction, respectively. Following the settings in PersonMAE~\cite{PersonMAE}, all models are pre-trained on LUPerson-4.2M~\cite{LU-Person}. For each framework, the DCS-Attention modules are incorporated into the backbone feature extractor to enable differentiable selection of informative channels in the computation of attention weights. The IB regularization, implemented via reducing the proposed IBB, is jointly optimized with the training loss for each of the self-supervised learning methods. It is observed in Table~\ref{tab:msmt_selfsup} that the RIB augmented models based on CLIP, DINO, and MAE consistently outperform the corresponding baseline models.
For instance, RIB-DCS-FB (CLIP-REID) outperforms CLIP-REID by $1.2\%$ in mAP for person Re-ID on Market1501, which demonstrates the general applicability of RIB in learning informative features within modern self-supervised feature learning frameworks.

\begin{table}[!htbp]
\centering
\caption{Performance comparison of self-supervised person Re-ID frameworks and their DCS-enhanced variants on Market1501. All feature backbones are pre-trained on LUPerson-4.2M~\cite{LU-Person}.}
 \resizebox{1\columnwidth}{!}{
\begin{tabular}{lcccc}
\toprule
{Method} & {Backbone}  & {mAP (\%)} & {Rank-1 (\%)} \\
\midrule
CLIP-REID~\cite{CLIP-REID} & TransReID  & 89.6 & 95.5 \\
\textbf{RIB-DCS-FB (CLIP-REID)} & RIB-DCS-FB (TransReID)  & \textbf{90.8} & \textbf{96.6} \\
\midrule
SOLIDER~\cite{SOLIDER} & TransReID & 92.5 & 96.0 \\
\textbf{RIB-DCS-FB (SOLIDER)} & RIB-DCS-FB (TransReID)  & \textbf{93.8} & \textbf{97.2} \\
\midrule
PersonMAE~\cite{PersonMAE} & TransReID  & 93.6 & 97.1 \\
\textbf{RIB-DCS-FB (PersonMAE)} & RIB-DCS-FB (TransReID)   & \textbf{94.3} & \textbf{97.5} \\
\bottomrule
\end{tabular}
}
\label{tab:msmt_selfsup}
\end{table}
\subsection{Ablation Study on the Effects of Differentiable Channel Selection for Attention Weights}
\label{sec:ablation_dcs}
In this section, we study the effectiveness of the DCS-Attention modules in RIB-DCS-FB and RIB-DCS-DNAS models. The study is performed using three baseline models, including MobileNetV2, TransReID, and FBNetV2-Large.
Since vanilla MobileNetV2 and FBNetV2-Large do not have self-attention modules, we build two baseline models that incorporate vanilla self-attention modules, leading to MobileNetV2+Self-Attention and FBNetV2-Large+Self-Attention.
We also build ablation models that add IBB in the training loss on TransReID, MobileNetV2+Self-Attention, and FBNetV2-Large+Self-Attention. The IBB is incorporated into the above models with vanilla self-attention modules to encourage the attention outputs to be more informative than those from the models trained without IBB.
The models with vanilla self-attention and IBB are denoted as MobileNetV2+Self-Attention+IBB, TransReID+IBB, and FBNetV2-Large+Self-Attention+IBB.
It is observed in Table~\ref{tab:ablation-DCS-loss} that reducing IBB in the training of the vanilla models and the ablation models with vanilla self-attention improves their performance for person Re-ID, which demonstrates the effectiveness of IBB in promoting the learning of informative attention outputs.
Moreover, the performance of the models with vanilla self-attention modules and IBB can be further improved by replacing the vanilla self-attention modules with the proposed DCS-Attention modules.
In particular, our RIB-DCS-FB and RIB-DCS-DNAS models outperform the vanilla models and the ablation study models by large margins in terms of mAP and Rank-1. For example, RIB-DCS-FB (TransReID) outperforms the vanilla model, TransReID, by $2.4\%$ in mAP. In addition, RIB-DCS-FB (TransReID) outperforms the ablation model TransReID+IBB by $1.1\%$ in mAP, demonstrating the effectiveness of the DCS-Attention module by learning informative attention weights.
\begin{table}[!htbp]
 \centering
  \caption{Ablation study on the effects of DCS-Attention. The comparisons are performed on Market1501.}
 \label{tab:ablation-DCS-loss}
\resizebox{0.9\columnwidth}{!}{
\begin{tabular}{lccccc}
\toprule
Model                   & FLOPs  & mAP  &Rank-1      \\
\midrule
MobileNetV2    &0.382 G  & 84.5& 93.9     \\
MobileNetV2+IBB   & 0.382 G &85.2 &94.6   \\
MobileNetV2+Self-Attention+IBB   & 0.395 G &85.6 &94.6   \\
\textbf{RIB-DCS-FB (MobileNetV2)}   & 0.406 G & \textbf{86.4} & \textbf{95.1}  \\
\midrule
TransReID    & 19.3 G  & 88.9 & 95.2      \\
TransReID+IBB   & 19.3 G & 90.2 & 96.0 \\
\textbf{RIB-DCS-FB (TransReID)}  &19.7 G & \textbf{91.3} & \textbf{97.0}  \\
\midrule
FBNetV2-Large & 1.235 G    & 87.2 &95.0 \\
FBNetV2-Large+IBB  & 1.235 G& 87.8 &95.5 \\
FBNetV2-Large+Self-Attention+IBB  & 1.288 G& 88.4 &95.5 \\
\textbf{RIB-DCS-DNAS (FBNetV2-Large)}  & 1.302 G    & \textbf{89.1} &\textbf{95.8}  \\
\bottomrule
\end{tabular}
}
\end{table}
\subsection{Visualization Analysis}
\label{sec:grad_cam}
In this section, we apply Grad-CAM~\cite{selvaraju2017grad} to visualize the effectiveness of RIB-DCS and RIB-CA models for person Re-ID. We use Grad-CAM to visualize the regions within the input images that influence the predictions of different models.
Fig.~\ref{fig:cam} illustrates the Grad-CAM visualization results for the vanilla TransReID, RIB-DCS-FB (TransReID)+DCS-Attention, and our RIB-DCS-FB (TransReID).
The Grad-CAM visualization results illustrate that the representative self-attention modules in TransReID usually miss the important parts of human figures, potentially hurting the performance in person Re-ID. In contrast, RIB-DCS-FB focuses more strongly on the important parts of human figures in the input that are informative for person Re-ID.
More Grad-CAM visualization results for RIB-DCS-FB (MobileNetV2) and the corresponding baseline models are illustrated in Fig.~\ref{fig:cam_sup}.

Fig.~\ref{fig:cam_pose} illustrates that the baseline model, TransReID, either mistakenly focuses on the background areas or misses the important parts of human figures affected by pose changes. In contrast, RIB-DCS-FB (TransReID), mostly highlights informative and identity-related body parts across both poses.
The additional Grad-CAM visualization results demonstrate that the proposed RIB method promotes the learning of more robust and informative representations under challenging conditions, such as occlusion and extreme pose variations.

Fig.~\ref{fig:RIB-CA_cam} illustrate the Grad-CAM heatmaps for TransReID, TransReID+SE, TransReID+CBAM, RIB-CA-FB (TransReID+SE), and  RIB-CA-FB (TransReID+CBAM).
The baseline models, TransReID+SE and TransReID+CBAM, incorporate the SE~\cite{hu2018squeeze} attention modules and the CBAM~\cite{CBAM} attention modules by inserting them after each transformer block in the backbone of TransReID. RIB-CA-FB (TransReID+SE) and RIB-CA-FB (TransReID+CBAM) apply RIB on the backbones of TransReID+SE and TransReID+CBAM.
The Grad-CAM visualization results illustrate that the representative channel attention modules, SE and CBAM, either mistakenly focus on the background areas or miss the important human body regions, potentially hurting the performance of person Re-ID. In contrast, RIB-CA-FB focuses more strongly on the person's body in the input image.

\begin{figure}[!hbtp]
\begin{center}
\includegraphics[width=0.95\columnwidth]{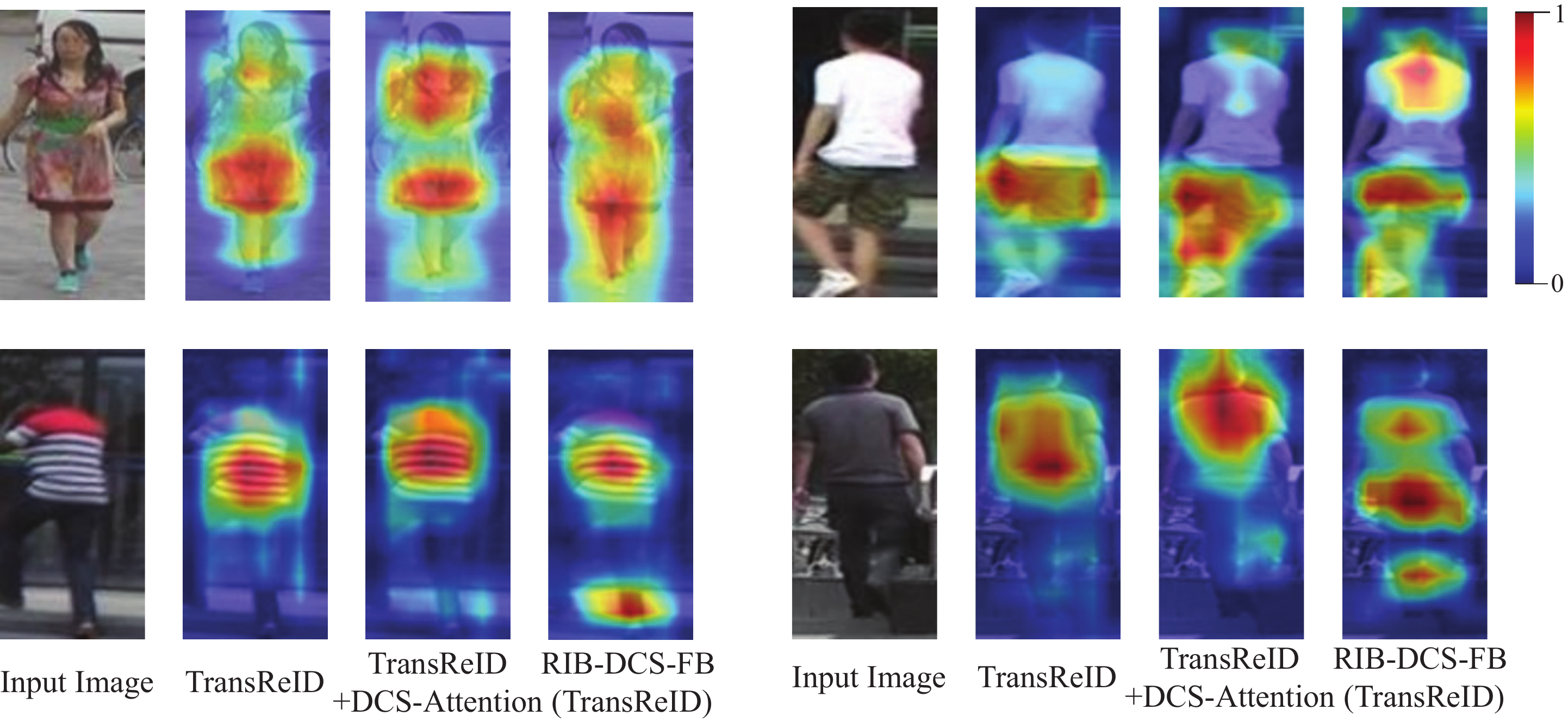}
\end{center}
\vspace{-4mm}
\caption{Grad-CAM visualization results for TransReID, TransReID+DCS-Attention, and RIB-DCS-FB (TransReID). TransReID+DCS-Attention replaces all the self-attention modules in TransReID with the DCS-Attention modules.}
\label{fig:cam}
\vspace{-2mm}
\end{figure}

\begin{figure}[!htbp]
\begin{center}
\includegraphics[width=0.95\columnwidth]{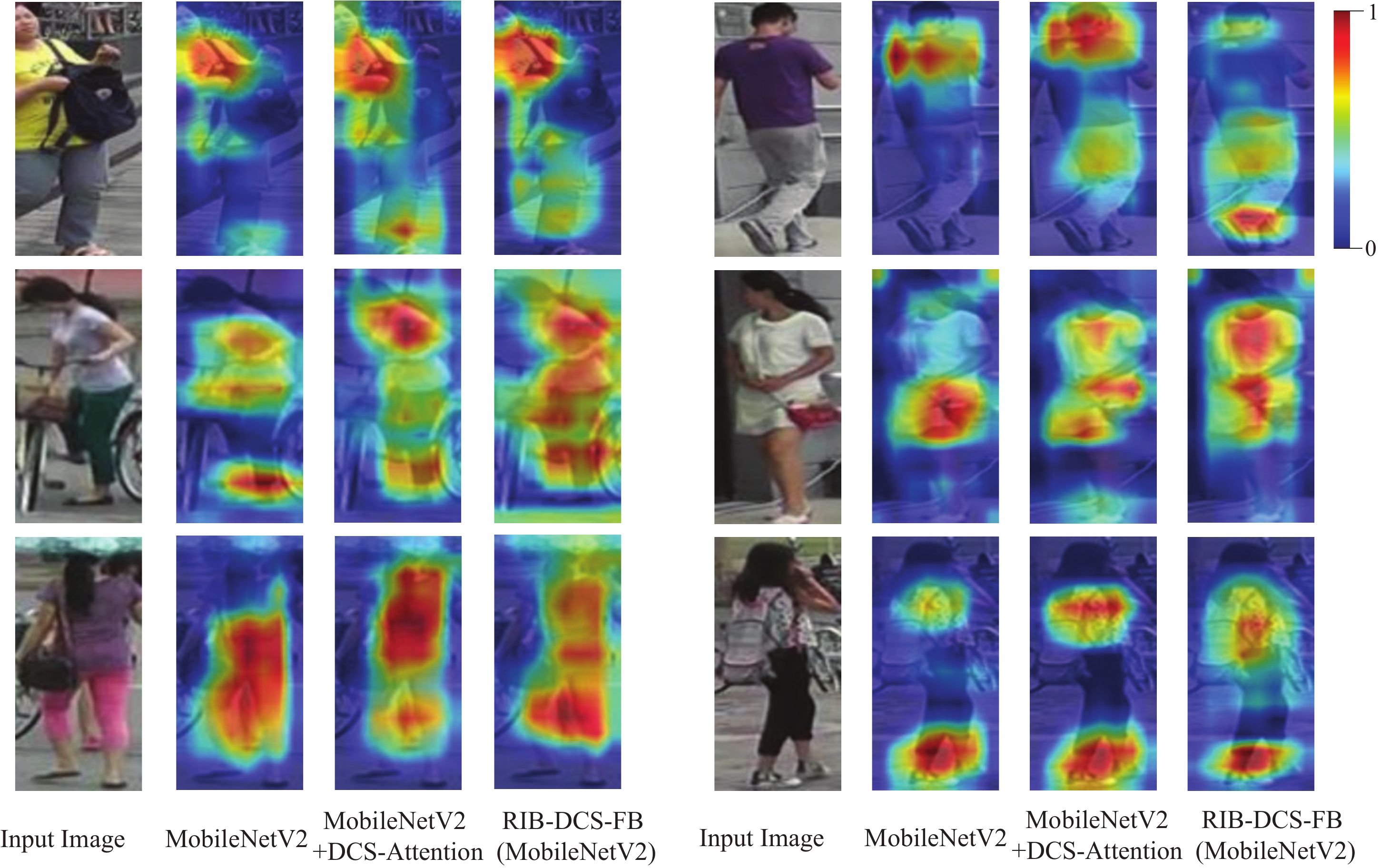}
\end{center}
\vspace{-4mm}
\caption{Grad-CAM visualization results for MobileNetV2, MobileNetV2+DCS-Attention, and RIB-DCS-FB (MobileNetV2). MobileNetV2+DCS-Attention inserts the DCS-Attention module after each of its residual blocks.}
\label{fig:cam_sup}
\vspace{-2mm}
\end{figure}

\begin{figure}[!htbp]
\begin{center}
\includegraphics[width=1\columnwidth]{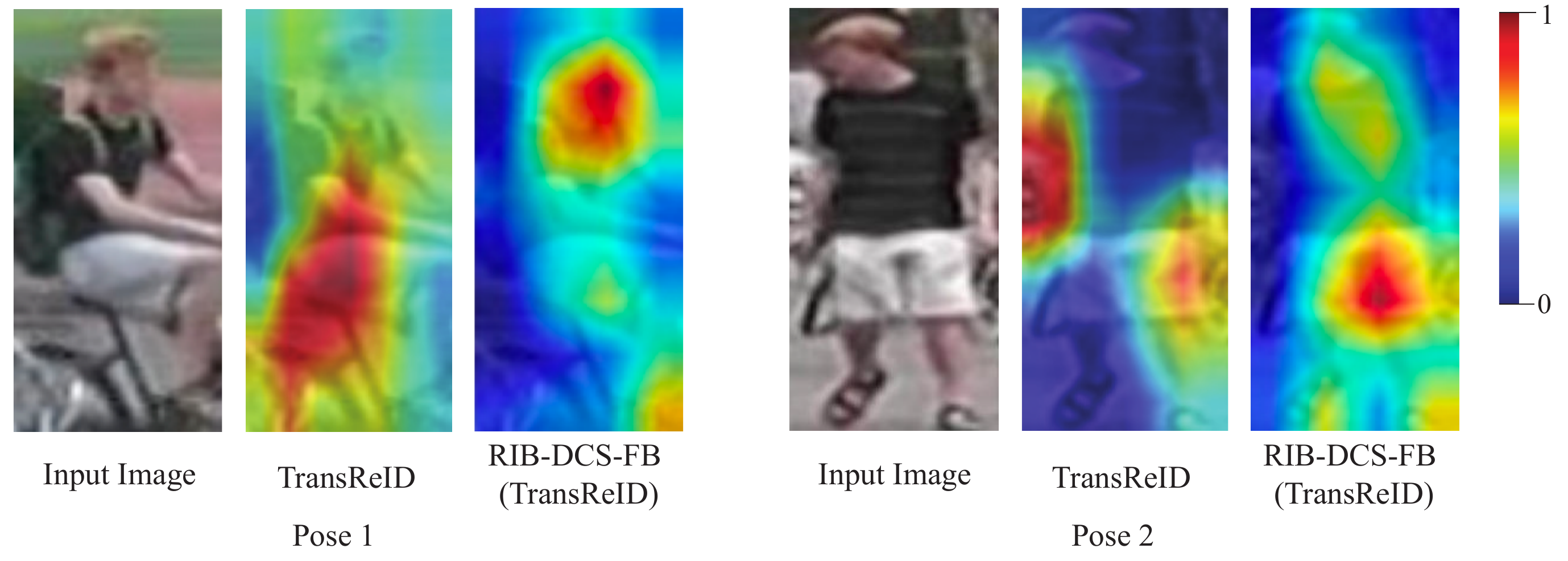}
\end{center}
\vspace{-4mm}
\caption{Grad-CAM visualization of a person under extreme pose variation in Markert1501 for TransReID and RIB-DCS-FB (TransReID).}
\label{fig:cam_pose}
\vspace{-5mm}
\end{figure}

\begin{figure}[!hbtp]
\begin{center}
\includegraphics[width=0.95\columnwidth]{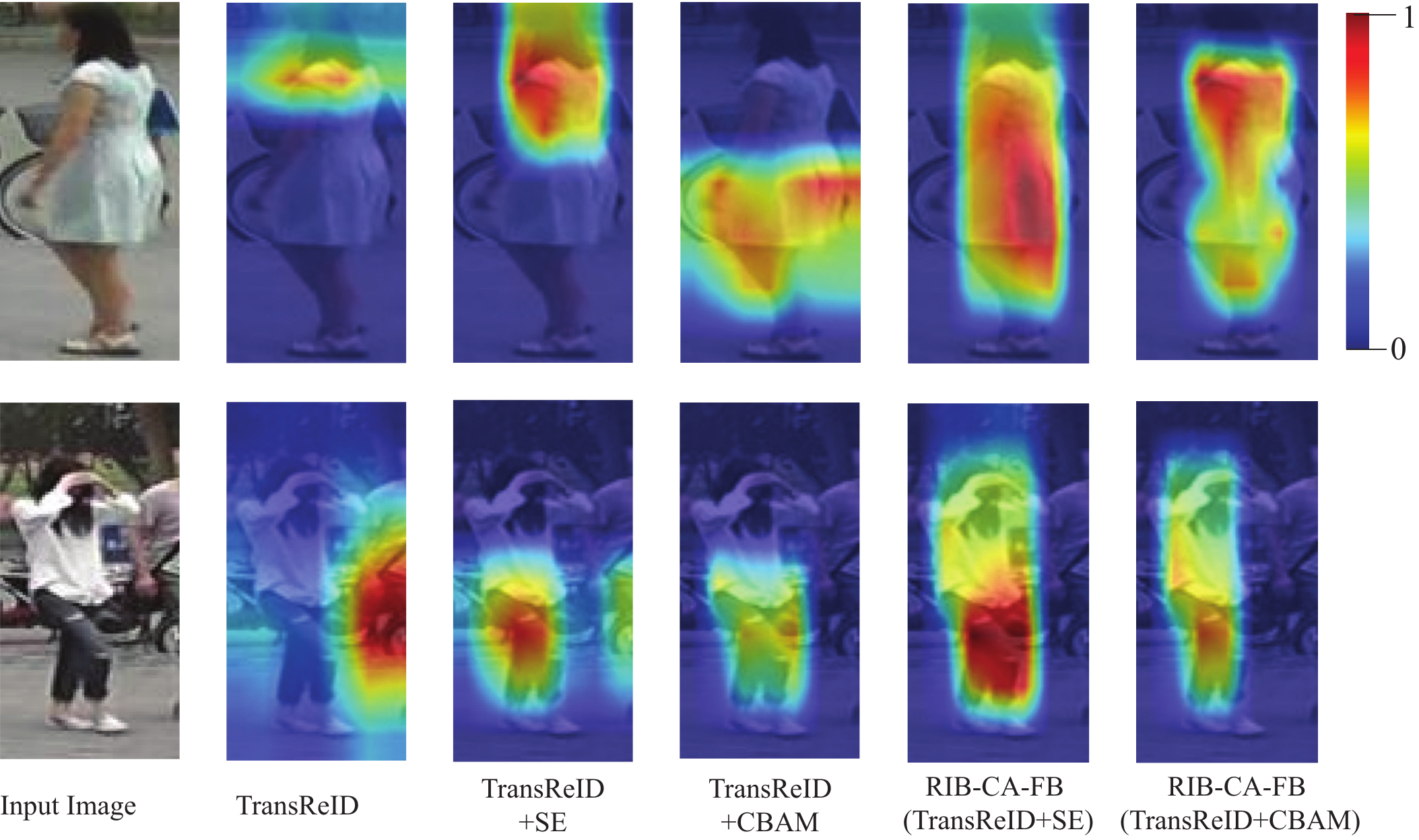}
\end{center}
\vspace{-4mm}
\caption{Grad-CAM visualization results for TransReID, TransReID+SE, TransReID+CBAM, RIB-CA-FB (TransReID+SE), and RIB-CA-FB (TransReID+CBAM). The SE attention modules and CBAM attention modules are inserted after each transformer block in TransReID+SE and TransReID+CBAM, respectively.}
\label{fig:RIB-CA_cam}
\vspace{-5mm}
\end{figure}

\begin{figure*}[!htbp]
    \centering
        \resizebox{1\textwidth}{!}{
            \includegraphics{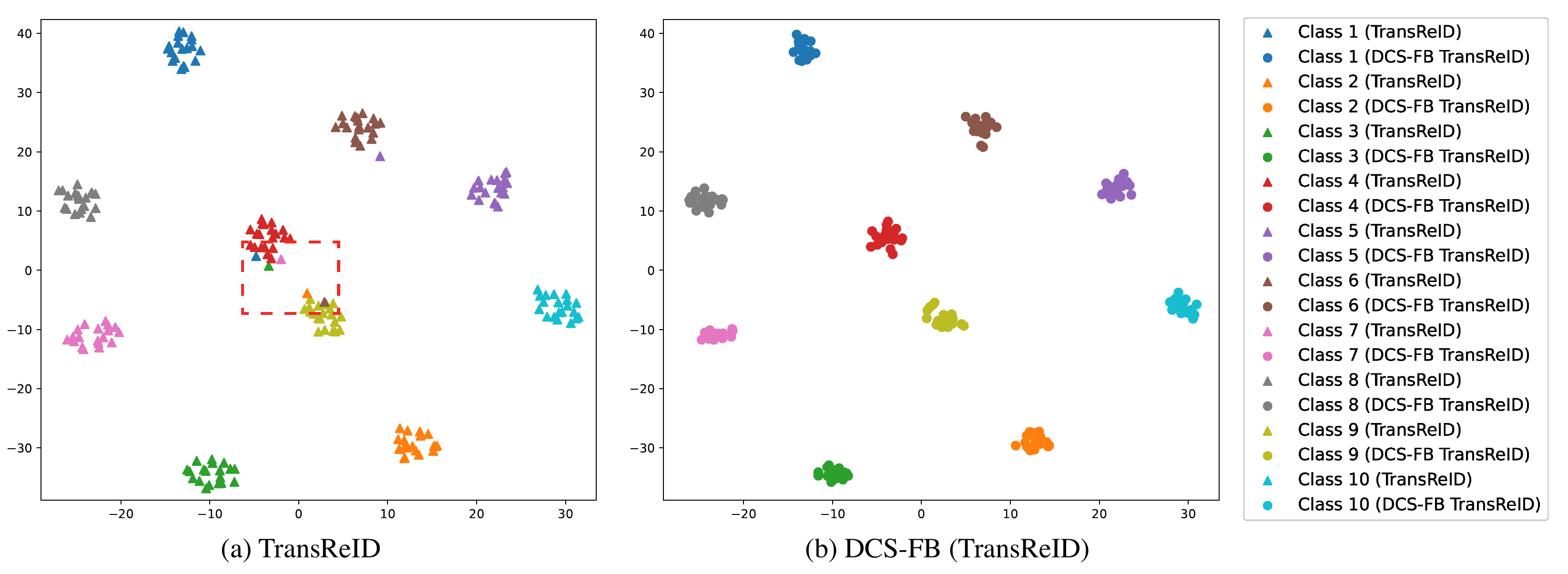}
        }
    \caption{t-SNE visualization analysis of the features learned by (a) TransReID and (b) RIB-DCS-FB (TransReID). The visualization is performed on the LDA-projected features for better discriminative analysis.
    The red dashed box in Figure (a) shows an example where features generated by TransReID for instances from multiple distinct classes exhibit substantial overlap.}
    \label{fig:tsne}
    \vspace{-3mm}
\end{figure*}

\subsection{t-SNE Visualization of Learned Features}
\label{sec:tsne}

To further assess the effectiveness of RIB in learning discriminative features, we perform the t-SNE visualization analysis on features learned by RIB-DCS-FB (TransReID) and the baseline model, TransReID. The visualization is conducted using images of $10$ randomly selected identities from the Market-1501 dataset, where each identity has $20$ images. Each identity is considered a distinct class.
The features are extracted from the output of the last transformer block before the classification head in the TransReID and RIB-DCS-FB (TransReID) models.
Standard normalization, which standardizes features to have zero mean and unit variance, is applied to the learned features for both TransReID and RIB-DCS-FB (TransReID) to ensure consistent scaling across dimensions. To prepare the features for discriminative analysis, we perform Linear Discriminant Analysis (LDA), which projects the original high-dimensional features into a lower-dimensional subspace that maximizes class separability~\cite{fisher1936use}, on the features generated by both TransReID and RIB-DCS-FB (TransReID). The number of LDA components, which is the dimension of the LDA subspace, is set to $9$.
Next, we perform t-SNE visualization on the LDA-projected features for both TransReID and RIB-DCS-FB (TransReID) to qualitatively assess the discriminative capacity of the learned representations.
As illustrated in Fig.~\ref{fig:tsne}, the features learned by RIB-DCS-FB (TransReID) exhibit better inter-class separability and increased intra-class compactness than those obtained without RIB-DCS-FB. For example, features generated by TransReID for instances from five distinct classes exhibit substantial overlap within the red dashed box in Fig.~\ref{fig:tsne}, which indicates limited inter-class separability and reduced intra-class compactness. In contrast, the features generated by RIB-DCS-FB (TransReID) for instances from different classes are more distinctly clustered and spatially separated from one another. This qualitative observation underscores the effectiveness of the informative channel selection in the proposed DCS-Attention in enhancing the discriminative capacity of the learned features.
To further support this observation, we compute the Fisher Score~\cite{DudaHS01} of the LDA-projected features, which quantitatively evaluates the discriminative capability of the features by measuring the ratio of inter-class variance to intra-class variance. The Fisher Score of the LDA-projected features of RIB-DCS-FB (TransReID) is $7.173$, which is $121\%$ higher than that of TransReID, which is $3.246$. This result quantitatively confirms that the proposed DCS-Attention enhances the feature discriminability, aligning well with the patterns observed in the t-SNE visualization.

\section{Illustration of the connection between the channel selection in DCS-Attention and the reduction of the IB loss}
\label{sec:informative_connection}

Fig.~\ref{fig:informative_connection} illustrates the connection between the channel selection and the IB loss.
In our DCS-Attention module, we propose to select informative feature channels so that more informative channels are used to learn more informative attention weights (please refer to our response to Question 1). By learning informative attention weights, DNNs assign higher attention weights to more informative feature tokens which contribute more to the discriminative task of person Re-ID. As a result, the aggregated feature tokens, which are the output of the DCS-Attention modules, are more informative for the discriminative task of person Re-ID than the aggregated feature tokens from vanilla self-attention modules. By replacing all the self-attention modules in DNNs with the DCS-Attention modules, the features learned by the DNNs are more informative for the discriminative task of person Re-ID. In order to reduce the IB loss, we propose an Information Bottleneck (IB) inspired channel selection for the attention weights, where the learned attention weights can be more informative by explicitly optimizing the IB loss for DNNs with DCS-Attention modules.

\begin{figure}[!htbp]
\begin{center}
\includegraphics[width=0.5\columnwidth]{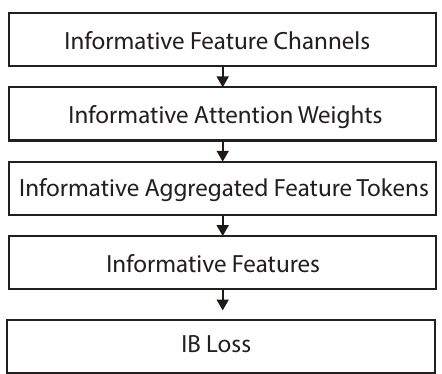}
\end{center}
\vspace{-4mm}
\caption{Illustration of the connection between the channel selection in DCS-Attention and the IB loss.
}
\label{fig:informative_connection}
\end{figure}

\section{Study on the Variational Upper Bound for the IB Loss (IBB) and Test Loss}
\label{sec:loss_figure}
Fig.~\ref{fig:loss} illustrates the test loss and the IBB during the training for RIB-DCS-FB (TransReID) and RIB-DCS-DNAS (FBNetV2). It is observed that the IBB for all the models decreases starting from the beginning of the training, and the test loss of a model trained with the IBB drops faster than that of the vanilla model.

\begin{figure}[!htb]
    \centering
        \resizebox{1\columnwidth}{!}{
            \includegraphics{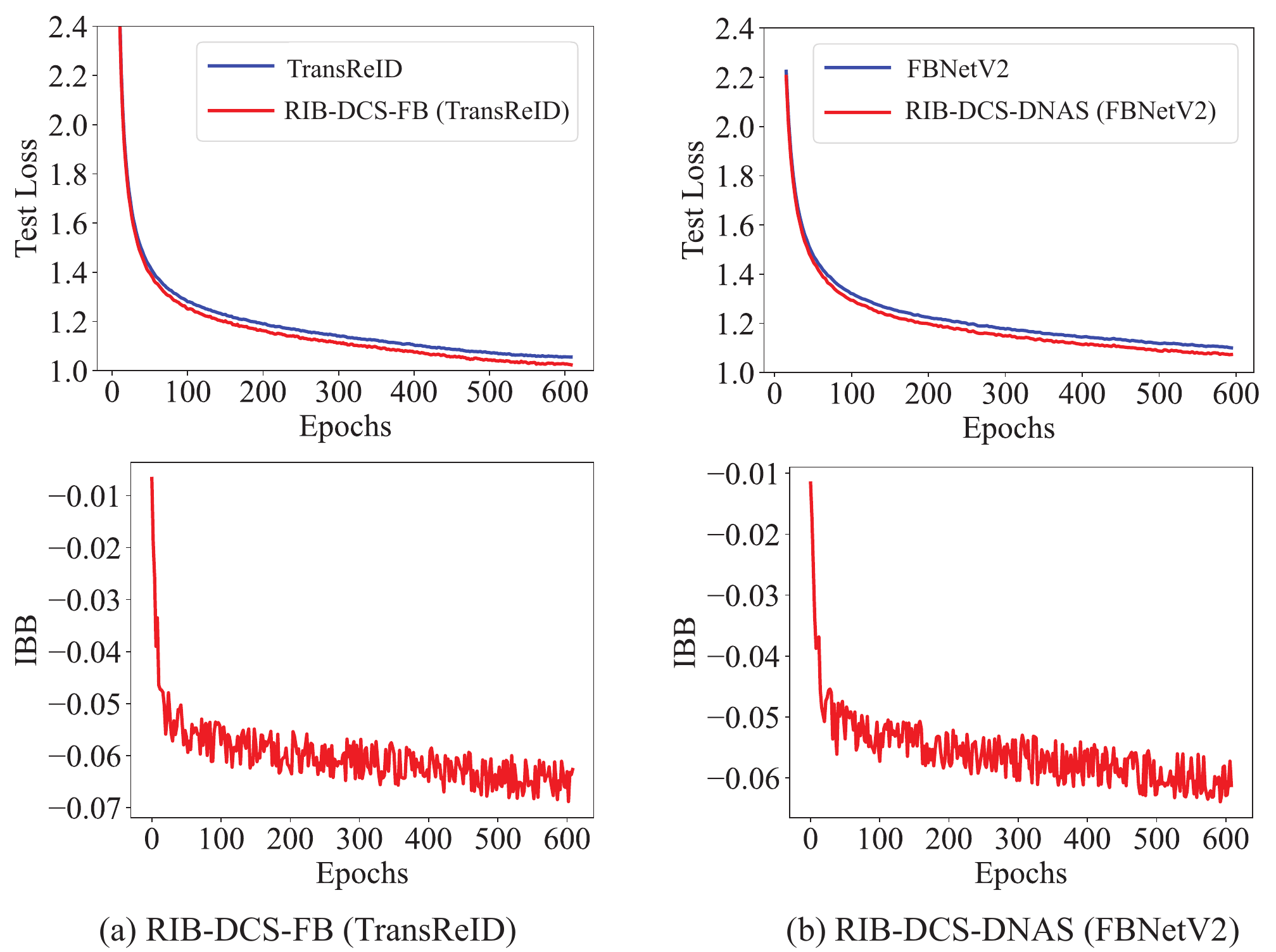}
        }
        \vspace{-2mm}
    \caption{IBB and test loss of RIB-DCS-FB (TransReID) and RIB-DCS-DNAS (FBNetV2) during the training.}
     \vspace{-5mm}
    \label{fig:loss}
\end{figure}

\section{Computational Complexity Analysis}
\label{sec:complexity}
Given samples $\set{(x_i, z_i)}_{i=1}^{n}$ drawn from the joint distribution $p_{\sigma}(x,z)=p_\sigma(z|x)p(x)$,
the goal of CLUB is to optimize the parameters of the predictive neural network $p_\sigma(z|x)$
such that the resulting joint model $p_\sigma(x,z)$ induces minimal mutual information between $x$ and $z$.
Here, $p_\sigma(z|x)$ serves as the main predictive model, such as an encoding model or a classification model, while $p(x)$ denotes the empirical data distribution.
In addition, $q_\theta(z|x)$ is a variational conditional distribution, also implemented as a neural network,
used to approximate $p(z|x)$ and to provide a differentiable estimator of mutual information.
The optimization alternates between updating $q_\theta(z|x)$ by maximizing the log-likelihood
$\mathcal{L}(\theta) = \frac{1}{n} \sum_{i=1}^{n} \log q_\theta(z_i|x_i)$ to improve the approximation accuracy, and updating $p_\sigma(x,z)$ by minimizing the upper bound for the mutual information, $I(x,z)$,
$\hat{I}_{\textup{vCLUB}} = \frac{1}{n} \sum_{i=1}^{n} U_i$,
where $U_i = \log q_\theta(z_i|x_i) - \frac{1}{n} \sum_{j=1}^{n} \log q_\theta(z_j|x_i)$.
Through this alternating process, the algorithm jointly learns an accurate variational estimator $q_\theta$
and a predictive model $p_\sigma$.
The complete training procedure is summarized in Algorithm~\ref{alg:vclub}.

\noindent Suppose the computation of $q_\theta(z|x)$ costs $T_q$ and that of $p_\sigma(z|x)$ costs $T_p$.
Computing $\cL(\theta) = \frac{1}{n}\sum_{i=1}^{n} \log q_\theta(z_i|x_i)$ requires $\Theta(n T_q)$ time.
The computation of vCLUB $\hat{I}_{\textup{vCLUB}} = \frac{1}{n}\sum_{i=1}^{n}\Big[\log q_\theta(z_i|x_i) - \frac{1}{n}\sum_{j=1}^{n}\log q_\theta(z_j|x_i)\Big]$
requires computing $[\log q_\theta(z_j|x_i)]$ for $i,j\in[n]$, which requires $\Theta(n^2 T_q)$ time.
Updating $p_\sigma$ for minimizing $\hat{I}_{\textup{vCLUB}}$ over the same $n$ inputs
adds $n T_p$ time.
Hence, one training epoch requires $\Theta(n^2 T_q+n T_q+n T_p) = \Theta(n^2 T_q+n T_p)$.

\begin{algorithm}[!htbp]
\caption{MI Minimization with vCLUB (Algorithm 1 in CLUB~\cite{CLUB})}
\label{alg:vclub}
\begin{algorithmic}[1]
\FOR{each training iteration}
\STATE Sample $\set{(x_i, z_i)}_{i=1}^{n}$ from $p_{\sigma}(x,z)$
\STATE Compute log-likelihood $\mathcal{L}(\theta) = \frac{1}{n}\sum_{i=1}^{n}\log q_\theta(z_i|x_i)$
\STATE Update $q_\theta(z|x)$ by maximizing $\mathcal{L}(\theta)$
\FOR{$i = 1$ to $n$}
\STATE $U_i = \log q_\theta(z_i|x_i) - \frac{1}{n}\sum_{j=1}^{n}\log q_\theta(z_j|x_i)$
\ENDFOR
\STATE Update $p_\sigma(x,z)$ by minimizing $\hat{I}_{\textup{vCLUB}} = \frac{1}{n}\sum_{i=1}^{n}U_i$
\ENDFOR
\end{algorithmic}
\end{algorithm}

\noindent\textbf{Computational Complexity of IBB.}~
Herein we analyze the computational complexity for computing the upper bound
$\textup{IBB}(\cW) =  \frac 1{n} \sum\limits_{i=1}^n \sum\limits_{a=1}^C \sum\limits_{y=1}^C
\indict{y_i = y} \phi(F_i, a) \log \pth{ \frac{\indict{y_i = y}}{p_y Q(F \in a| Y=y)} }$
for the IB loss $I(F(\cW), X)$ for a comparison with CLUB~\cite{CLUB}.
Let $T_0$ denote the complexity of a forward and backward computation of the model predicting $F_i(\cW)$ for one sample.
For each epoch, the computation of the soft assignments $\phi(F_i,a)$ for $i\in [n]$ and $a\in [C]$ requires $\Theta(nCT_0)$ time.

Once the soft assignment matrix $\phi(F_i, a)$ for all $i \in [n]$ and $a \in [C]$ has been obtained at the $t$-th epoch, the conditional distribution matrix
$Q^{(t)}(F \in a \mid Y = y) \in \mathbb{R}^{C \times C}$ can be efficiently computed following Algorithm~\ref{alg:Q_computation}.
We denote by $Q[a, y] = Q^{(t)}(F \in a \mid Y = y)$ the $(a, y)$-th entry of $Q$, representing the conditional probability that a feature $F$ belongs to cluster $a$ given class label $y$.
Each entry $Q[a, y]$ is computed by aggregating the soft assignment values $\phi(F_i, a)$ over all $i\in [n]$ whose class label $y_i = y$,
followed by normalization with respect to the total number of samples in that class.
The accumulation step (lines 2–7 in Algorithm~\ref{alg:Q_computation}) requires $\Theta(nC)$ time, while the normalization step (lines 8–12 in Algorithm~\ref{alg:Q_computation}) requires $\Theta(C^2)$ time.
Consequently, the overall computational complexity of computing $Q$ is $\Theta(nC+C^2)$.

\begin{algorithm}[!htbp]
\caption{Efficient Computation of $Q^{(t)}(F \in a \mid Y = y)$}
\label{alg:Q_computation}
\begin{algorithmic}[1]
\REQUIRE The precomputed soft assignments $\phi(F_i,a)$ for $i \in [n]$ and $a \in [C]$, the class labels $\set{y_i}_{i=1}^{n}$, and the number of clusters $C$. Here $\set{F_i}_{i=1}^n$ are computed at the $t$-th epoch in Algorithm 1 of the main paper.
\ENSURE Conditional distribution matrix $Q \in \mathbb{R}^{C \times C}$

\STATE Initialize $Q \leftarrow \mathbf{0}^{C \times C}$ and count vector $M \leftarrow \mathbf{0}^{C}$.
\FOR{$i = 1$ to $n$}
\FOR{$a = 1$ to $C$}
\STATE $Q[a, y_i] \leftarrow Q[a, y_i]+\phi(F_i, a)$
\ENDFOR
\STATE $M[y_i] \leftarrow M[y_i]+1$
\ENDFOR
\FOR{$y = 1$ to $C$}
\FOR{$a = 1$ to $C$}
\STATE $Q[a, y] \leftarrow Q[a, y] / M[y]$
\ENDFOR
\ENDFOR
\RETURN $Q$
\end{algorithmic}
\end{algorithm}

Let $\cI_y = \set{i\in[n]|y_i = y}$ be the index set of the training samples from the class $y$ for $y\in[C]$. Then, $\textup{IBB}(\cW) = \frac 1{n} \sum\limits_{i=1}^n \sum\limits_{a=1}^C \sum\limits_{y=1}^C
\indict{y_i = y} \phi(F_i, a) \log \pth{ \frac{\indict{y_i = y}}{p_y Q(F \in a| Y=y)} }=  \frac 1{n} \sum\limits_{a=1}^C \sum\limits_{y=1}^C
\sum\limits_{i\in \cI_y }  \phi(F_i, a) \log \pth{ \frac{1}{p_y Q(F \in a| Y=y)} }$. As a result, the computation of $\textup{IBB}(\cW)$ given  $\phi(F_i,a)$ for $i\in [n]$ and $a\in [C]$ and $Q$ takes $C\times\sum\limits_{y=1}^C|\cI_y| = \Theta(nC)$ time because $\sum\limits_{y=1}^C|\cI_y| = n$. The total computation cost for $\textup{IBB}(\cW)$ is $\Theta(nCT_0+nC+C^2+nC) = \Theta(nCT_0 +C^2 )$.


\section{Difference of Our Hard Channel Selection in DCS-Attention from Existing Channel Attention Methods and Its Advantage}
\label{sec:sup_hard-channel-selection}
The key difference of our hard channel selection in DCS-Attention from existing channel attention methods, such as SE \cite{hu2018squeeze}, CBAM~\cite{CBAM}, MCA~\cite{MCA}, is that the channel selection mechanism selects an informative subset of the channels to compute the attention weights in a hard selection manner, rather than the soft reweighting scheme used in existing channel attention methods, including SE, CBAM, and MCA. Furthermore, the channels selected by DCS-Attention are token-dependent, that is, different tokens can use different channels to compute the attention weights. Such a mechanism allows for adaptive attention weights, in contrast to existing channel attention methods which use all the feature channels to compute the attention weights.

\begin{table*}[!htbp]
 \centering
  \caption{Comparison between RIB-DCS and ablation models that replace the channel selection mechanism in DCS-Attention with channel attention modules, including SE~\cite{hu2018squeeze}, CBAM~\cite{CBAM}, and MCA~\cite{MCA}. SE (Self-Attention), CBAM (Self-Attention), and MCA (Self-Attention) denote self-attention modules that apply SE, CBAM, or MCA to reweight the channels of the query and key when computing attention weights. TransReID~\cite{he2021transreid} is used as the backbone for all the models in this table. }
\label{tab:ablation_channel_selection}
\resizebox{0.8 \textwidth}{!}{
\begin{tabular}{c|cccccccc}
\toprule

Dataset    & Model           & FLOPs  & mAP  &Rank-1 & IBB  &$I(F,X)$ &$I(F,Y)$  & IB Loss  \\
\midrule
\multirow{10}{*}{Market1501} &TransReID& 19.3 G  & 88.9 & 95.2 &0.059 & 0.071 &0.091 & -0.020     \\
& TransReID+SE (Self-Attention)   & 19.8 G & 89.5 &  95.3 & 0.055& 0.070& 0.092& -0.023\\
& TransReID+CBAM (Self-Attention)  &20.0 G & 89.2 & 95.2  & 0.060 & 0.070 &0.091 & -0.021\\
& TransReID+MCA (Self-Attention)   &20.3 G & 89.5 &  95.2 & 0.054 & 0.071&0.092 & -0.023\\
& \textbf{TransReID+DCS-Attention}   &19.7 G &\textbf{90.4} & \textbf{95.8}& \textbf{0.050} & \textbf{0.066} & \textbf{0.094} &\textbf{-0.028}  \\
\cmidrule{2-9}
&TransReID+IBB & 19.3 G & 90.2 & 96.0 & 0.040 & 0.062 & 0.105 & -0.043 \\
& RIB-SE (Self-Attention)  & 19.8 G &90.4 & 95.8 & 0.040 & 0.059 &0.103 & -0.044   \\
& RIB-CBAM (Self-Attention)   &20.0 G &90.1 &  95.7 & 0.041 & 0.063 &0.102 & -0.039   \\
& RIB-MCA (Self-Attention)   &20.3 G & 90.3 & 96.2  & 0.035& 0.060 &0.108 & -0.048   \\
& \textbf{RIB-DCS-FB (TransReID)}   &19.7 G & \textbf{91.3} & \textbf{97.0} &\textbf{0.027} & \textbf{0.045} & \textbf{0.112} & \textbf{-0.067}     \\
\midrule

\multirow{10}{*}{DukeMTMC}
&TransReID & 19.3 G  & 82.0 & 90.7 & 0.068 & 0.060 & 0.090 & -0.030 \\
& TransReID+SE (Self-Attention)   &19.8 G &82.3 &  90.8 & 0.065 & 0.059 & 0.091 & -0.032 \\
& TransReID+CBAM (Self-Attention)   & 20.0 G&82.1 & 90.7 & 0.066 & 0.058 & 0.091 & -0.033 \\
& TransReID+MCA (Self-Attention)   & 20.3 G& 82.3  &90.6  & 0.067&0.060 &0.092 &  -0.032\\
& \textbf{TransReID+DCS-Attention}   &19.7 G & \textbf{83.1}& \textbf{91.0} & \textbf{0.060} & \textbf{0.048} & \textbf{0.093} & \textbf{-0.045} \\
\cmidrule{2-9}
&TransReID+IBB & 19.3 G  & 83.3 & 91.1 & 0.033 & 0.046 & 0.101 & -0.055  \\
& RIB-SE (Self-Attention)  & 19.8 G &83.4  & 91.2 & 0.035 & 0.044 & 0.099 & -0.055 \\
& RIB-CBAM (Self-Attention)   &20.0 G &83.2  & 91.0 & 0.036 & 0.046 & 0.099 & -0.053 \\
& RIB-MCA (Self-Attention)   &20.3 G &83.2  & 91.2 & 0.034 & 0.045 & 0.101 & -0.056 \\
& \textbf{RIB-DCS-FB (TransReID)}   &19.7 G & \textbf{84.3} & \textbf{92.0} & \textbf{0.022} & \textbf{0.035} & \textbf{0.105} & \textbf{-0.070} \\

\midrule
\multirow{10}{*}{MSMT17}
&TransReID& 19.3 G  & 67.4 & 85.3 & 0.068 & 0.059 & 0.075 & -0.016 \\
& TransReID+SE (Self-Attention)  &19.8 G & 67.8 & 85.5 & 0.066 & 0.056 & 0.074 & -0.018 \\
& TransReID+CBAM (Self-Attention)  &20.0 G & 67.7 & 85.3 & 0.067 & 0.054 & 0.073 & -0.019 \\
& TransReID+MCA (Self-Attention)   & 20.3 G& 67.5& 85.5 &0.069  & 0.058 &0.074 & -0.016 \\
& \textbf{TransReID+DCS-Attention}   &19.7 G & 68.8 & 86.3 & \textbf{0.062} & \textbf{0.048} & \textbf{0.075} & \textbf{-0.027} \\
\cmidrule{2-9}
&TransReID+IBB & 19.3 G  &68.8 & 86.5& 0.033 & 0.052 & 0.089 & -0.037 \\
& RIB-SE (Self-Attention)  & 19.8 G & 69.2 &86.5 & 0.034 & 0.052 & 0.091 & -0.039 \\
& RIB-CBAM (Self-Attention)   &20.0 G & 68.8& 86.4 & 0.036 & 0.051 & 0.089 & -0.038 \\
& RIB-MCA (Self-Attention)   &20.3 G & 68.8& 86.5 & 0.036 & 0.054 & 0.091 & -0.037 \\
& \textbf{RIB-DCS-FB (TransReID)}   &19.7 G & \textbf{70.2} & \textbf{87.1} & \textbf{0.030} & \textbf{0.040} & \textbf{0.093} & \textbf{-0.053} \\

\bottomrule
\end{tabular}
}
\end{table*}

The DNNs with DCS-Attention modules under the general framework of RIB are referred to as the RIB-DCS models.
The output of existing self-attention~\cite{he2021transreid, zhang2020relation} is specified by $\textup{Output} = \sigma(QK^{\top})V$, where $Q,K,V \in \RR^{N \times D}$ denote the query, key, and value, respectively, with $N$ tokens and channel dimension $D$.
The operator $\sigma(\cdot)$ (e.g., Softmax) yields attention weights $A=\sigma(QK^\top)\in\RR^{N\times N}$, enabling weighted aggregation over the feature tokens $V$.
The hard channel selection mechanism in DCS-Attention uses only selected informative channels, which are the selected columns of $Q$ and the corresponding selected rows of $K^{\top}$, to compute $A = \sigma(QK^{\top})$, as illustrated in Fig.~3(a) of the main paper. Informative $A$ assigns larger attention weights to semantically more relevant tokens, such as the human body regions in Fig.~1(c)–(d) of the main paper, while suppressing occlusions and background, in contrast to the existing self-attention method SPT~\cite{tan2024occluded}.

Learning informative attention weights by our DCS-Attention through such a hard channel selection scheme can be explained with the Information Bottleneck (IB) principle. Let $X$ denote input features, $F$ the learned features by the DNNs, and $Y$ the ground-truth identity labels. The IB principle aims to increase the mutual information between $F$ and $Y$ while decreasing that between $F$ and $X$. Thus, it encourages reducing the IB loss $I(F, X) - I(F, Y)$, where $I(\cdot,\cdot)$ denotes mutual information modelling the correlation between the two input variables. By reducing the IB loss, informative attention weights are expected to reduce the correlation between $F$ and $X$ while increasing the correlation between $F$ and $Y$. In this way, learned features $F$ are less correlated with the input containing background or occlusions, thus mitigating the adverse effects of background or occlusions in the input, and they are more correlated with the person identity which is the goal of person Re-ID.

\begin{figure*}[!htb]
\begin{center}
 \includegraphics[width=0.9\linewidth]{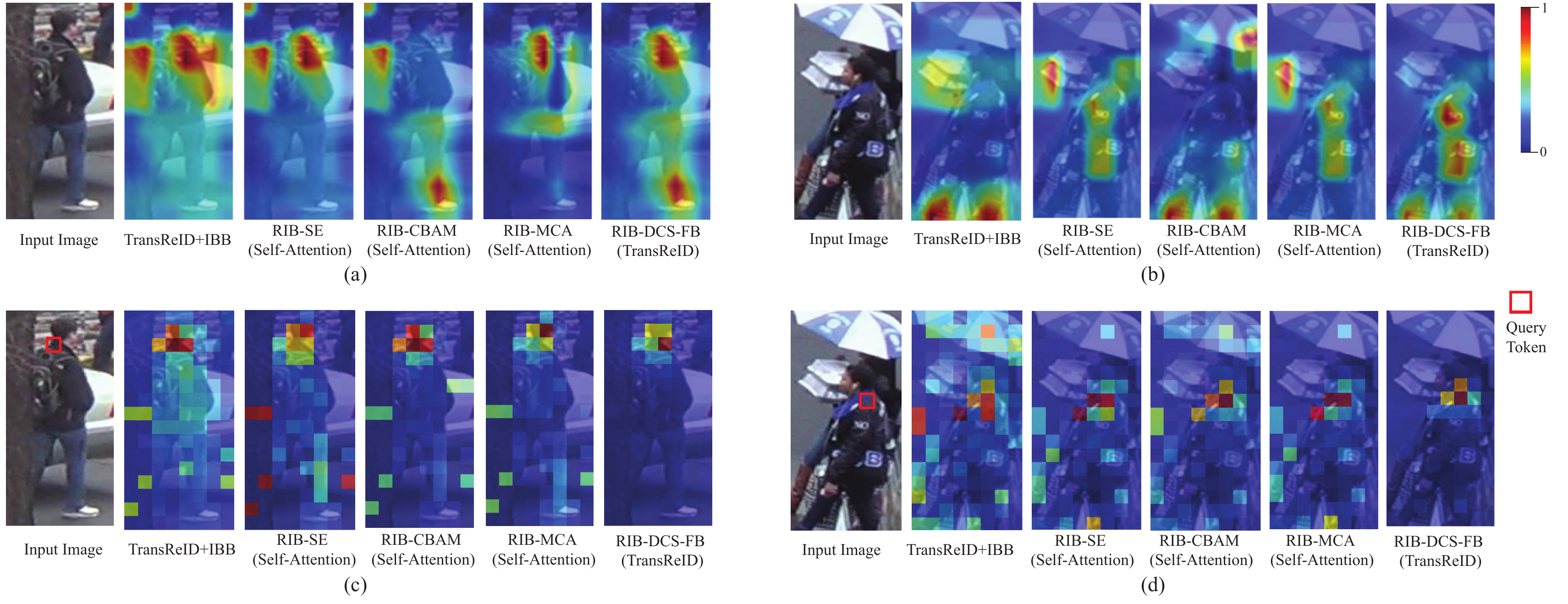}
\end{center}
\vspace{-3mm}
\caption{Figures (a)-(b) illustrate the Grad-CAM visualization for two images from DukeMTMC \cite{duke}, for TransReID+IBB, RIB-SE (Self-Attention), RIB-CBAM (Self-Attention), RIB-MCA (Self-Attention), and RIB-DCS-FB (TransReID).
Figures (c)-(d) illustrate the heatmaps of the attention weights corresponding to a query token computed from the first transformer block for the above models. The query token for both examples is selected from the boundary of the neck of the person in the input images, which is critical for identifying a person.
}
\vspace{-3mm}
\label{fig:gradcam_dcs_ablation}
\end{figure*}

The channels of $Q$ and $K$ encode all the information in the input image, so that existing channel attention methods, such as SE, CBAM, and MCA, which use all the feature channels, inevitably suffer from the background or occlusions in the input even after a reweighting process on the feature channels, potentially hurting the performance of person Re-ID. In a strong contrast, the hard channel selection in DCS-Attention only uses a learned subset of channels to compute the attention weights, so that only a selected subset of input information is used to compute the attention weights and the correlation between the learned features $F$ and the input $X$, $I(F,X)$, is decreased. Such a hard channel selection mechanism is inspired by the IB principle. In particular, the decreased $I(F,X)$ mitigates the adverse effects of background or occlusions in the input, while maintaining a reasonable correlation between $F$ and $Y$ by standard training with cross-entropy loss involving $Y$. DCS-Attention is further enhanced by the adaptive channel selection mechanism where different tokens are allow to different channels to compute
the attention weights.
To demonstrate the advantage of DCS-Attention models,
including TransReID+DCS-Attention and RIB-DCS-FB (TransReID), we build three ablation self-attention modules denoted as SE (Self-Attention), CBAM (Self-Attention), and MCA (Self-Attention), which perform soft channel reweighting on all the feature channels in $Q$ and $K$ in self-attention, using SE, CBAM, and MCA, respectively.
The models that replace the self-attention modules in TransReID with SE (Self-Attention), CBAM (Self-Attention), and MCA (Self-Attention) are denoted as TransReID+SE (Self-Attention), TransReID+CBAM (Self-Attention), and TransReID+MCA (Self-Attention), respectively. The model named TransReID+DCS-Attention replaces all the self-attention modules in TransReID with DCS-Attention, which is trained with the standard cross-entropy and triplet loss without our proposed IBB to explicitly reduce the IB loss. In addition, we build three ablation RIB models referred to as RIB-SE (Self-Attention), RIB-CBAM (Self-Attention), and RIB-MCA (Self-Attention), which replace the DCS-Attention modules in RIB-DCS with the SE (Self-Attention), CBAM (Self-Attention), and MCA (Self-Attention) modules. The IBB is added to the training loss of these three ablation RIB models to explicitly reduce the IB loss. We also evaluate a baseline model, TransReID+IBB, which adds the IBB to the training loss of the vanilla TransReID
without DCS-Attention.

The advantage of the hard channel selection in DCS-Attention models over the ablation self-attention models, including TransReID+SE (Self-Attention), TransReID+CBAM (Self-Attention), and TransReID+MCA (Self-Attention), is evidenced in Table~\ref{tab:ablation_channel_selection} in Section~E of the supplementary, which is also copied to this summary.
In particular,
Table~\ref{tab:ablation_channel_selection} shows that the mutual information $I(F, X)$ of
TransReID+DCS-Attention is smaller than that of all three ablation self-attention models, while keeping similar mutual information $I(F, Y)$ thanks to the training by the cross-entropy loss involving $Y$.
We remark that the IB loss can already be reduced by DCS-Attention without explicitly reducing the IB loss in the training process, mostly due to the reduced $I(F,X)$. As a result, TransReID+DCS-Attention outperforms the three ablation self-attention models in terms of both mAP and Rank-1.
For example, TransReID+DCS-Attention achieves an average $I(F,X)$ of $0.054$
across the three datasets in Table~\ref{tab:ablation_channel_selection}, which is $12.9\%$ lower than that of the ablation self-attention model,  TransReID+SE (Self-Attention), that is $0.062$.
As a result, TransReID+DCS-Attention outperforms TransReID+SE (Self-Attention) by $0.90\%$ in average mAP and $0.83\%$ in average Rank-1.
Our RIB-DCS-FB models, termed RIB-DCS-FB (TransReID) in Table~\ref{tab:ablation_channel_selection}, further explicitly reduce the IB loss with RIB by adding the IBB to the training loss. It can be observed from Table~\ref{tab:ablation_channel_selection} that compared to TransReID+DCS-Attention, the person Re-ID performance of RIB-DCS-FB (TransReID) is further improved, and the IB loss of RIB-DCS-FB is reduced to an even lower level, which indicates the effectiveness of the proposed RIB. For instance, RIB-DCS-FB (TransReID) reduces the average IB loss of TransReID+DCS-Attention by $0.03$ on the three datasets, leading to a $1.17\%$ improvement in average mAP and a $1.0\%$ improvement in average Rank-1. The best of the three RIB ablation models renders only marginal improvement over that of TransReID+IBB, e.g., by $0.23\%$ in average mAP. In contrast, RIB-DCS-FB (TransReID) outperforms TransReID+IBB by $1.17\%$ in average mAP across the three datasets in Table~\ref{tab:ablation_channel_selection}, which demonstrates the advantages of the hard channel selection in DCS-Attention over existing channel attention methods under the general framework of RIB.

Fig.~\ref{fig:gradcam_dcs_ablation}(c)-(d) in Section~E of the supplementary, which is also copied to this summary, illustrates the heatmaps of the attention weights of TransReID+IBB, RIB-DCS (Self-Attention), RIB-SE (Self-Attention), RIB-CBAM (Self-Attention), and our RIB-DCS-FB (TransReID).  With the three ablation RIB models, it is observed that even feature tokens from the occlusion and the background areas of the input image receive high attention weights, which are greater than $0.5$, for a query token from the boundary of the neck of the person.
For instance, all three ablation RIB models assign high attention weights, which are greater than $0.5$, to regions in the middle of the tree that occlude the person in Fig.~\ref{fig:gradcam_dcs_ablation}(c).
In contrast, our proposed RIB-DCS model, RIB-DCS-FB (TransReID), assigns higher attention weights to informative and semantically relevant feature tokens, which are also located around the neck of the person in the input image.
The Grad-CAM visualization in Fig.~\ref{fig:gradcam_dcs_ablation}(a)-(b) illustrates that the three ablation RIB models either mistakenly focus on the occlusion and background areas or miss the important human body regions, potentially hurting the performance of person Re-ID. In contrast, our proposed RIB-DCS model, RIB-DCS-FB (TransReID), mostly focuses on the important human body regions in the input images.
In addition, it is illustrated in Fig.~\ref{fig:gradcam_dcs_ablation}(a)-(b) that TransReID+IBB, the model without DCS-Attention, still focuses on background or occlusion areas in the input images even with the explicit reduction of the IB loss, evidencing the necessity of DCS-Attention.

\section{Proof of Theorem 3.1}
\label{sec:proofs}
\begin{lemma}\label{lemma:I-X-tildeX-upper-bound}
Let $\Prob{X \in b}  = \sum_{i=1}^n \indict{y_i = b}/n \defeq p_b$  for every $b \in [C]$, then
\bal\label{eq:I-X-tildeX-upper-bound}
\resizebox{1\columnwidth}{!}
{$
I(F, X) \le \frac 1{n} \sum\limits_{i=1}^n \sum\limits_{a=1}^C \sum\limits_{b=1}^C
\indict{y_i = b} \phi(F_i, a) \log \pth{ \frac{\indict{y_i = b}}{p_b} \phi(F_i, a)}.
$}
\eal
\end{lemma}
\begin{proof}
Define $ \sum_{i=1}^n \indict{y_i = b} \defeq n_b$ for every $b \in [C]$, then
 $\Prob{X \in b} =  n_b/n \defeq p_b$. We first have
\bal\label{eq:I-X-tildeX-upper-bound-seg-1}
&\Prob{F \in a, X \in b} = \Prob{X \in b} \cdot \Prob{F \in a \longmid X \in b} \nonumber \\
&= \frac{n_b}{n} \cdot \frac{1}{n_b} \sum\limits_{i \colon y_i = b} \phi(F_i, a) \nonumber \\
&=\frac{1}{n}  \sum\limits_{i=1}^n  \indict{y_i = b} \phi(F_i, a).
\eal
It then follows from (\ref{eq:I-X-tildeX-upper-bound-seg-1}) and the log sum inequality that

\noindent\resizebox{1\columnwidth}{!}{
\begin{minipage}{1\columnwidth}

\bal\label{eq:I-X-tildeX-upper-bound-seg-2}
\allowdisplaybreaks
&I(F, X) \nonumber \\
&= \sum\limits_{a=1}^C \sum\limits_{b=1}^C
\Prob{F \in a, X \in b} \log{\frac{\Prob{F \in a, X \in b}}
{\Prob{F \in a}\Prob{X \in b}}}  \nonumber \\
&\le
\frac 1 {n} \sum\limits_{i=1}^n \sum\limits_{a=1}^C \sum\limits_{b=1}^C
\indict{y_i = b} \phi(F_i, a) \log\pth{\indict{y_i = b} \phi(F_i, a)}
\nonumber \\
&- \frac 1 {n} \sum\limits_{i=1}^n \sum\limits_{a=1}^C \sum\limits_{b=1}^C
\indict{y_i = b} \phi(F_i, a) \log \pth{\phi(F_i,a) p_b} \nonumber \\
&=\frac 1{n} \sum\limits_{i=1}^n \sum\limits_{a=1}^C \sum\limits_{b=1}^C
\indict{y_i = b} \phi(F_i, a) \log \pth{ \frac{\indict{y_i = b}}{p_b}}.
\eal
\vspace{1mm}
\end{minipage}
}
\end{proof}
\begin{lemma}\label{lemma:I-tildeX-Y-lower-bound}
\bal\label{eq:I-tildeX-Y-lower-bound}
I(F, Y)
&\ge \frac 1n \sum\limits_{a=1}^C \sum\limits_{y=1}^C
 \sum\limits_{i=1}^n \phi(F_i,a) \indict{y_i = y} \log{Q(F \in a| Y=y)}
\eal
\end{lemma}
\begin{proof}
Let $Q(F | Y)$ be a variational distribution. We have
\noindent\resizebox{1\columnwidth}{!}{
    \begin{minipage}{1\columnwidth}
\bal\label{eq:I-tildeX-Y-lower-bound-seg2}
&I(F, Y) \nonumber \\
&= \sum\limits_{a=1}^C \sum\limits_{y=1}^C
\Prob{F \in a, Y = y} \log{\frac{\Prob{F \in a, Y = y}}
{\Prob{F \in a}\Prob{Y = y}}} \nonumber \\
&= \sum\limits_{a=1}^C \sum\limits_{y=1}^C
\Prob{F \in a, Y = y} \log{\frac{\Prob{F \in a|Y = y}Q(F \in a | Y=y)}
{\Prob{F \in a} Q(F \in a | Y=y)}} \nonumber \\
& \ge \sum\limits_{a=1}^C \sum\limits_{y=1}^C
\Prob{F \in a, Y = y} \log{\frac{\Prob{F \in a|Y = y}}
{Q(F \in a | Y=y)}} \nonumber \\
&+ \sum\limits_{a=1}^C \sum\limits_{y=1}^C
\Prob{F \in a, Y = y} \log{\frac{Q(F \in a | Y=y)}
{\Prob{F \in a}}} \nonumber \\
&=\textup{KL}\pth{P(F | Y) \middle\| Q(F | Y) }\nonumber \\
&+ \sum\limits_{a=1}^C \sum\limits_{y=1}^C
\Prob{F \in a, Y = y} \log{\frac{Q(F \in a | Y=y)}
{\Prob{F \in a}}} \nonumber \\
&\ge \sum\limits_{a=1}^C \sum\limits_{y=1}^C
\Prob{F \in a, Y = y} \log{\frac{Q(F \in a| Y=y)}
{\Prob{F \in a}}} \nonumber \\
&= \sum\limits_{a=1}^C \sum\limits_{y=1}^C
\Prob{F \in a, Y = y} \log{Q(F \in a | Y=y)}
+ H\pth{P(F)} \nonumber \\
&\ge \sum\limits_{a=1}^C \sum\limits_{y=1}^C
\Prob{F \in a, Y = y} \log{Q(F \in a| Y=y)}
\nonumber \\
&\ge \frac 1n \sum\limits_{a=1}^C \sum\limits_{y=1}^C
 \sum\limits_{i=1}^n \phi(F_i,a) \indict{y_i = y} \log{Q(F \in a| Y=y)}.
\eal
        \vspace{1mm}
    \end{minipage}
}
\end{proof}

\begin{proof}[Proof of Theorem 3.1]
Equation (1) in Theorem 3.1 of the main paper follows from $\textup{IB}(\cW) = I(F(\cW),X)
-I(F(\cW),Y)$, the upper bound for $I(F(\cW),X)$ in Lemma~\ref{lemma:I-X-tildeX-upper-bound} and the lower bound for $I(F(\cW),Y)$
in Lemma~\ref{lemma:I-tildeX-Y-lower-bound}.
\end{proof}

\subsection{Computation of $Q^{(t)}(F | Y)$}
\label{sec:Q-compute}
The variational distribution $Q^{(t)}(F | Y)$ can be computed by
\bal\label{eq:Q-computation}
\resizebox{1\linewidth}{!}{
$Q^{(t)}(F \in a | Y = y) =
\Prob{F \in a | Y = y}
=\frac {\sum\limits_{i=1}^n \phi(F_i,a) \indict{y_i=y}}{\sum\limits_{i=1}^n
\indict{y_i=y}}$}.
\eal

\end{appendices}



\vspace{-3mm}
\bibliographystyle{IEEEtran}
\bibliography{ref}

\end{document}